%% file: main.tex
\documentclass[sigconf, nonacm]{acmart}

\usepackage{nicefrac}
\usepackage{siunitx}
\usepackage{array,framed}
\usepackage{booktabs}
\usepackage{
	color,
	float,
	epsfig,
	wrapfig,
	graphics,
	graphicx,
	subcaption
}
\usepackage{textcomp}
\usepackage{setspace}
\usepackage{latexsym,fancyhdr,url}
\usepackage{enumerate}
\usepackage[noend,ruled,linesnumbered]{algorithm2e}
\usepackage{algpseudocode}

\usepackage{graphics}
\usepackage{xparse} 
\usepackage{xspace}
\usepackage{multirow}
\usepackage{csvsimple}
\usepackage{balance}
\usepackage{threeparttable}

\usepackage{
	tikz,
	pgfplots,
	pgfplotstable
}
\usepackage{hyperref}

\usepackage{mathtools,}
\usepackage{fancyhdr}

\begin{document}
	\title{DPSUR: Accelerating Differentially Private Stochastic Gradient Descent Using Selective Update and Release}
	
	\author{Jie Fu$^{1, 2}$, Qingqing Ye$^2$, Haibo Hu$^2$, Zhili Chen$^{1, *}$, Lulu Wang$^1$, Kuncan Wang$^1$, Xun Ran$^2$}
	\affiliation{%
		\institution{$^1$Shanghai Key Laboratory of Trustworthy Computing, East China Noraml University, China}
	}
	\affiliation{%
		\institution{$^2$Department of Electronic and Information Engineering, The Hong Kong Polytechnic University, China}
	}
	\email{jie.fu@stu.ecnu.edu.cn, qqing.ye@polyu.edu.hk, haibo.hu@polyu.edu.hk, zhlchen@sei.ecnu.edu.cn}
    \email{luluwang@stu.ecnu.edu.cn, 10204804424@stu.ecnu.edu.cn, qi-xun.ran@connect.polyu.hk}
	
	\begin{abstract}
		Machine learning models are known to memorize private data to reduce their training loss, which can be inadvertently exploited by privacy attacks such as model inversion and membership inference. To protect against these attacks, differential privacy (DP) has become the de facto standard for privacy-preserving machine learning, particularly those popular training algorithms using stochastic gradient descent, such as DPSGD.
		Nonetheless, DPSGD still suffers from severe utility loss due to its slow convergence. This is partially caused by the random sampling, which brings bias and variance to the gradient, and partially by the Gaussian noise, which leads to fluctuation of gradient updates.
		
		Our key idea to address these issues is to apply selective updates to the model training, while discarding those useless or even harmful updates. Motivated by this, this paper proposes DPSUR, a \underline{D}ifferentially \underline{P}rivate training framework based on \underline{S}elective \underline{U}pdates and \underline{R}elease, where the gradient from each iteration is evaluated based on a validation test, and only those updates leading to convergence are applied to the model. As such, DPSUR ensures the training in the right direction and thus can achieve faster convergence than DPSGD. The main challenges lie in two aspects --- privacy concerns arising from gradient evaluation, and gradient selection strategy for model update. To address the challenges, DPSUR introduces a clipping strategy for update randomization and a threshold mechanism for gradient selection. 
		Experiments conducted on MNIST, FMNIST, CIFAR-10, and IMDB datasets show that DPSUR significantly outperforms previous works in terms of convergence speed and model utility.
		
	\end{abstract}
	
	\maketitle

    \thispagestyle{fancy}
    \fancyhead{}
    \chead{\textcolor{red}{\large This paper has been accepted by 50th International Conference on Very Large Data Bases (VLDB 2024)}}
    \renewcommand{\headrulewidth}{0.5pt}

     \pagestyle{empty}

    	\renewcommand\thefootnote{}\footnote{\noindent
		 * Corresponding author}
	
	\input{intro}    
	\input{preliminary}

	\input{methodology1}
	\input{methodology2}
	\input{analysis}
	\input{results}

	\input{related}

	\input{conclusion}

	\begin{acks}
		This work was support by the Natural Science Foundation of
		Shanghai (Grant No. 22ZR1419100), CAAI-Huawei MindSpore Open Fund (Grant No. CAAIXSJLJJ-2022-005A), National Natural Science Foundation of China Key Program (Grant No. 62132005), National Natural Science Foundation of China (Grant No: 92270123 and 62372122), and the Research Grants Council, Hong Kong SAR, China (Grant No:  15209922, 15208923 and 15210023).
	\end{acks}
	
	

        \appendix
	\input{appendix}

\end{document}

%% file: intro.tex
\section{Introduction}
\label{sec:intro}

In the past decade, deep learning techniques have achieved remarkable success in many AI tasks, such as image recognition~\cite{lundervold2019overview,zhou2021review,suzuki2017overview}, text analysis~\cite{chen2019gmail}, and recommendation systems~\cite{wu2016using}.
However, even though the training data are not published, adversaries may still learn them by analyzing the model parameters. For example, the contents of training data can be inverted from the models~ \cite{zhu2019deep,fredrikson2015model,nasr2019comprehensive,wang2019beyond,phong2017privacy}, or the membership information of the training set can be inferred~\cite{song2017machine,melis2019exploiting}.
This is of particular concern in those applications which involve sensitive and personal data, such as medical imaging and finance. Recent legislations such as EU's General Data Privacy Regulation (GDPR) and California Consumer Privacy Act have mandated machine learning practitioners to take legal responsibility for protecting private data~\cite{cummings2018role}.

One of the state-of-the-art paradigms to prevent privacy disclosure in machine learning is differential privacy (DP)~ \cite{dwork2014algorithmic}. Many works \cite{abadi2016deep,carlini2019secret,feldman2020does,xiang2019differentially, AnttiKoskela2018LearningRA,shokri2015privacy,jayaraman2019evaluating,stock2022defending} have demonstrated that by adding proper DP noise in the training phase, the resulted machine learning models can prevent unintentional leakage of private training data, such as membership inference attacks.


Among these works, the seminal work by Abadi et al. \cite{abadi2016deep} proposes differentially private stochastic gradient descent (DPSGD) as the training algorithm. In DPSGD, each iteration involves four main steps: (i) randomly selecting a batch of samples using Poisson sampling, (ii) computing and clipping the gradient for each sample, (iii) adding random Gaussian noise to each gradient based on a privacy loss analysis, and (iv) updating the model weights using the average noisy gradients in the batch. In these steps, random Gaussian noise and Poisson sampling are the main reasons to cause slower convergence than conventional SGD~\cite{wei2022dpis, xu2020adaptive, yu2020not, zhang2020variance}.

\begin{enumerate}
	
\item
\textbf{Gaussian Noise.} In DPSGD, Gaussian noise is added to the gradient in each iteration to satisfy differential privacy. However, the noise scale can be forbiddingly large, leading to inaccurate gradient estimation and poor optimization, especially when it is close to convergence~\cite{abadi2016deep}.
	
\item
\textbf{Poisson Sampling.}
{Random sampling in SGD is usually implemented as epoch partitioning in practice. 
	However, in DPSGD random sampling is popularly implemented as Possion sampling for its privacy amplification effect~\cite{IlyaMironov2019RnyiDP}. Since Possion sampling may lead to bias and variance in gradient estimation, resulting in unstable and slow convergence to optimize the objective function \cite{katharopoulos2018not,bottou2009curiously,gurbuzbalaban2021random}, we attribute the performance issue of DPSGD partially to random sampling.}
\end{enumerate}



Although there are a few recent works which optimize random sampling or Gaussian noise separately~\cite{wei2022dpis,lee2018concentrated,yu2019differentially}, as shown by the blue trace in Figure~\ref{figure:1}, all these works still blindly update the model, whether or not the loss is improved in this iteration. This issue is particularly eminent when the training process approaches convergence and the magnitude of loss is small.


\begin{figure}[htb]
	\begin{center}
		\includegraphics[width=0.8\linewidth]{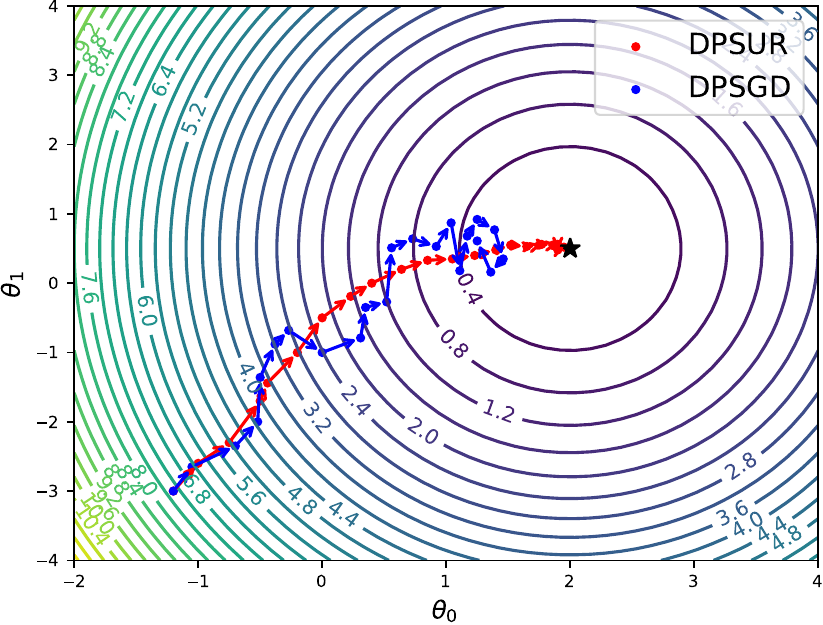}
		\caption{Trajectory visualization of DPSGD and DPSUR schemes on a linear regression model.}
		\label{figure:1}
	\end{center}
\end{figure}

In this paper, we propose DPSUR, a {\underline D}ifferentially {\underline P}rivate Stochastic Gradient Descent training framework based
on {\underline S}elective {\underline U}pdate and {\underline R}eleases. In essence, DPSUR only executes differentially private gradient descent if the gradient is in the correct direction to decrease loss and thus leads to a better model, as shown by the red trace in Figure.~\ref{figure:1}. This idea is inspired by simulated annealing (SA) \cite{metropolis1953equation,kirkpatrick1983optimization}, where the objective function value (loss) is treated as ``energy'', and the differentially private model parameters obtained in each iteration is a random ``solution''. As such, we use the difference between the current and previous iteration's loss $\Delta E$ as a criterion for accepting or rejecting model parameters.

Selecting those ``correct'' updates in DPSGD is non-trivial. First, the evaluation of gradient $\Delta E$ needs to access a training set, which consumes additional privacy budgets. Second, how to set a threshold for selective update is essential as it directly affects the training performance. These two issues are closely interleaved. To address the first issue, we propose a clipping strategy to $\Delta E$, which employs a minimal clipping bound to minimize the perturbation error. For the second issue, we devise a threshold mechanism for update selection by optimizing the utility gain from selective update. 

In addition, we also propose an optimization that selectively releases gradients during multiple iterations, which further reduces the privacy budget consumption. 
In summary, our contributions are as follows:

\begin{itemize}
    \item We propose a new DPSUR framework for differentially private deep learning, which uses a validation test to select the model updates and ensures the gradient updates are in the correct direction. To our knowledge, this is the first work to apply selective update for model training under differential privacy.
    \item A clipping strategy and a threshold mechanism are devised to guarantee DP privacy and optimize utility gain, which enables a faster convergence and better accuracy than DPSGD.
    \item Furthermore, we propose Gaussian mechanism with selective release to reduce privacy budget consumption across iterations.
    \item Rigorous theoretical analysis has guaranteed differential privacy of DPSUR, together with a comprehensive empirical evaluation on four public datasets. The results confirm that DPSUR outperforms the state-of-the-art solutions in terms of model accuracy.
\end{itemize}


The rest of the paper is organized as follows. Section~\ref{sec:prelim} introduces the preliminary knowledge. In Section~\ref{sec:methodology}, we present our method of selective update. Section~\ref{sec:methodology2} shows the selective release mechanism and the complete algorithm DPSUR. Privacy analysis is conducted in Section~\ref{sec:Analysis} and the experimental results are presented in Section~\ref{sec:eval}. Section~\ref{sec:related} introduces related work, followed the conclusion in Section~\ref{sec:conclusion}. 

%% file: preliminary.tex
\section{Preliminary Knowledge} \label{sec:prelim}

\subsection{Differential Privacy} \label{sec:prelim-dp}
Differential privacy is a rigorous mathematical framework that formally defines data privacy. It requires that a single entry in the input dataset must not lead to statistically significant changes in the output \cite{dwork2006calibrating,dwork2011firm,dwork2014algorithmic} if differential privacy holds.

\begin{definition}
({\bf Differential Privacy~\cite{dwork2014algorithmic}}). The randomized mechanism $A$ provides ($\epsilon$,  $\delta$)-Differential Privacy (DP), if for any two neighboring datasets $D$ and $D'$ that differ in only a single entry, $\forall$S $\subseteq$ Range($A$),
\begin{equation}
{\rm Pr}(A(D) \in S) < e^{\epsilon} \times {\rm Pr}(A(D') \in S) + \delta.
\end{equation}
\end{definition}

Here, $\epsilon > 0$ controls the level of privacy guarantee in the worst case. The smaller $\epsilon$, the stronger the privacy level is. The factor $\delta > 0$ is the failure probability that the property does not hold. In practice, the value of $\delta$ should be negligible~\cite{zhu2017differential,papernot2018scalable}, particularly less than $\frac{1}{|D|}$.


By adding random noise, we can achieve differential privacy for a function $f: \mathcal{X}^n \rightarrow \mathbb{R}^d$ according to Definition 2.1. The $l_k$-sensitivity determines how much noise is needed and is defined as follow.

\begin{definition} ({\bf $\mathbf{l_k}$-Sensitivity\cite{dwork2006calibrating}}) For a function $f: \mathcal{X}^n \rightarrow \mathbb{R}^d$, we define its $l_k$ norm sensitivity (denoted as $\Delta_k f$) over all neighboring datasets $x, x^{'} \in \mathcal{X}^n$ differing in a single sample as
\begin{align}
    \text{sup}_{x, x^{'} \in \mathcal{X}^n} ||f(x) - f(x^{'})||_k \leq \Delta_k f.
\end{align}
\end{definition}

In this paper, we focus on $l_2$ sensitivity, i.e., $|| \cdot ||_2$. Additionally, the following Lemma~\ref{lem:prelim-post} ensures the privacy guarantee of post-processing operations.

\begin{lemma}[{\bf Post-processing~\cite{dwork2014algorithmic}}]\label{lem:prelim-post}
Let $\mathcal{M}$ be a mechanism satisfying $(\epsilon, \delta)$-DP. Let $f$ be a function whose input is the output of $\mathcal{M}$. Then $f(\mathcal{M})$ also satisfies $(\epsilon,\delta)$-DP.
\end{lemma}

\subsection{Rényi Differential Privacy}
Rényi differential privacy (RDP) is a relaxation of $\epsilon$-differential privacy, which is defined on Rényi divergence as follows.

\begin{definition}({\bf Rényi Divergence \cite{van2014renyi}}) Given two probability distributions $P$ and $Q$, the Rényi divergence of order $\alpha > 1$ is: 
\begin{align}
D_{\alpha}(P \| Q)=\frac{1}{\alpha-1} \ln \mathbf{E}_{x \sim Q}\left[\left(\frac{P(x)}{Q(x)}\right)^{\alpha}\right],
\end{align}
where $\mathbf{E}_{x \sim Q}$ denotes the excepted value of $x$ for the distribution $Q$, $P(x)$, and $Q(x)$ denotes the density of $P$ or $Q$ at $x$ respectively.
\end{definition}


\begin{definition}({\bf Rényi Differential Privacy (RDP) \cite{mironov2017renyi}}) For any neighboring datasets $x, x^\prime \in \mathcal{X}^n$, a randomized mechanism $\mathcal{M}: \mathcal{X}^n \rightarrow \mathbb{R}^{d}$ satisfies $(\alpha, R)$-RDP if
\begin{align}
    D_{\alpha}(\mathcal{M}(x) || \mathcal{M}(x^\prime)) \leq R.
\end{align}
\end{definition}

The following Definition~\ref{definition:Gaussian mechanism of RDP} provides a formal definition of Gaussian mechanism, and a formal RDP guarantee by it.

\begin{definition}({\bf RDP of Gaussian mechanism~\cite{mironov2017renyi}})\label{definition:Gaussian mechanism of RDP}
	Assuming $f$ is a real-valued function, and the sensitivity of $f$ is $\mu$, the Gaussian mechanism for approximating $f$ is defined as
	\begin{align} \label{equ:Gaussian}
	\mathbf{G}_{\sigma} f(D)=f(D)+N\left(0, \mu^2\sigma^{2}\right),
	\end{align}
	where $N(0, \mu^2\sigma^{2})$ is normally distributed random variable with standard deviation $\mu\sigma$ and mean $0$. Then the Gaussian mechanism with noise $\mathbf{G}_{\sigma}$ satisfies $(\alpha,\alpha / 2\sigma^2)-RDP$.
	
\end{definition}

The following Lemma~\ref{lem:conversion} defines the standard form for converting $(\alpha, R)$-RDP to ($\epsilon$,  $\delta$)-DP.

\begin{lemma}\label{lem:conversion}
({\bf Conversion from RDP to DP~\cite{balle2020hypothesis}}). if a randomized mechanism $f : D \rightarrow \mathbb{R}$  satisfies $(\alpha,R)$-RDP ,then it satisfies$(R+\ln ((\alpha-1) / \alpha)-(\ln \delta+ \ln \alpha) /(\alpha-1), \delta)$-DP for any $0<\delta<1$.
\end{lemma}


\subsection{Deep Learning with Differential Privacy}
Differentially Private Stochastic Gradient Descent (DPSGD) 
is a widely-adopted training algorithm for deep neural networks with differential privacy guarantees. 
Specifically, in each iteration $t$, a batch of tuples $\mathcal{B}_t$ is sampled from $D$ with a fixed probability $\frac{b}{|D|}$, where $b$ is the batch size. After computing the gradient of each tuple $x_i \in \mathcal{B}_t$ as $g_t(x_i) = \nabla_{\theta_i} L(\theta_i,x_i)$, where $\theta_i$ is model parameter for the i-th sample, DPSGD clips each per-sample gradient according to a fixed $\ell_{2}$ norm (Equation~\eqref{eq:clipping}).

\begin{align}
	\begin{split}
		\overline{g}_t\left(x_{i}\right)
		& = \textbf{Clip}(g_t\left(x_{i}\right);C) \label{eq:clipping} \\
		& = g_t\left(x_{i}\right) \Big/ \max \Big(1, 		\frac{\left\|g_t\left(x_{i}\right)\right\|_{2}}{C}\Big).
	\end{split}
\end{align}

In this way, for any two neighboring datasets, the sensitivity of the query $\sum_{i\in \mathcal{B}_t} g(x_i)$ is bounded by $C$. Then, it adds Gaussian noise scaling with $C$ to the sum of the gradients when computing the batch-averaged gradients:
\begin{equation}\label{eq:add noise}
\tilde{g}_t = \frac{1}{b}\left(\sum_{i \in \mathcal{B}_t} \overline{g}_t\left(x_{i}\right)+\mathcal{N}\left(0, \sigma^{2} C^{2} \mathbf{I}\right)\right),
\end{equation}
where $\sigma$ is the noise multiplier depending on the privacy budget. Last, the gradient descent is performed based on the batch-averaged gradients. Since initial models are randomly generated and independent of the sample data, and the batch-averaged gradients satisfy the differential privacy, the resulted models also satisfy the differential privacy due to the post-processing property.

\textbf{Privacy accounting.} Three factors determine DPSGD's privacy guarantee --- the noise multiplier $\sigma$, the sampling ratio $\frac{b}{|D|}$, and the number of training iterations $T$.
In reality, given the privacy parameters $(\epsilon, \delta)$, we can set appropriate values for these three hyper-parameters to optimize the performance.
The privacy calibration process is performed using a privacy accountant: a numerical algorithm providing tight upper bounds for the given $(\epsilon, \delta)$ as a function of the hyper-parameters~\citep{abadi2016deep}, which in turn can be combined with numerical optimization routines to optimize one hyper-parameter given the other two.
In this work we use the RDP~\cite{mironov2017renyi} for privacy accounting. In practice, given $\sigma$, $\delta$ and $b$ at each iteration, we select $\alpha$ from $\left\{2,3,...,64\right\}$ to determine the smallest $\epsilon$.


%% file: methodology1.tex
\section{DPSUR: DP Training Framework with Selective Updates and Release}
\label{sec:methodology}
In this section, we present our proposed framework DPSUR, with an overview in Section~\ref{sec:overview}. Then two key components of DPSUR, namely, minimal clipping strategy and threshold mechanism,are introduced in Sections~\ref{sec:cliping} and \ref{sec:threshold}, respectively.

\subsection{Overview}
\label{sec:overview}
As aforementioned, DPSUR does not directly accept the model updates from each iteration due to the influence of random sampling and Gaussian noise. Therefore, we first calculate the loss of the generated model in each iteration, and then compare it with that from the last iteration to determine whether or not to accept the model update.

\begin{figure}[htb]
	\begin{center}
		\includegraphics[width=0.8\linewidth]{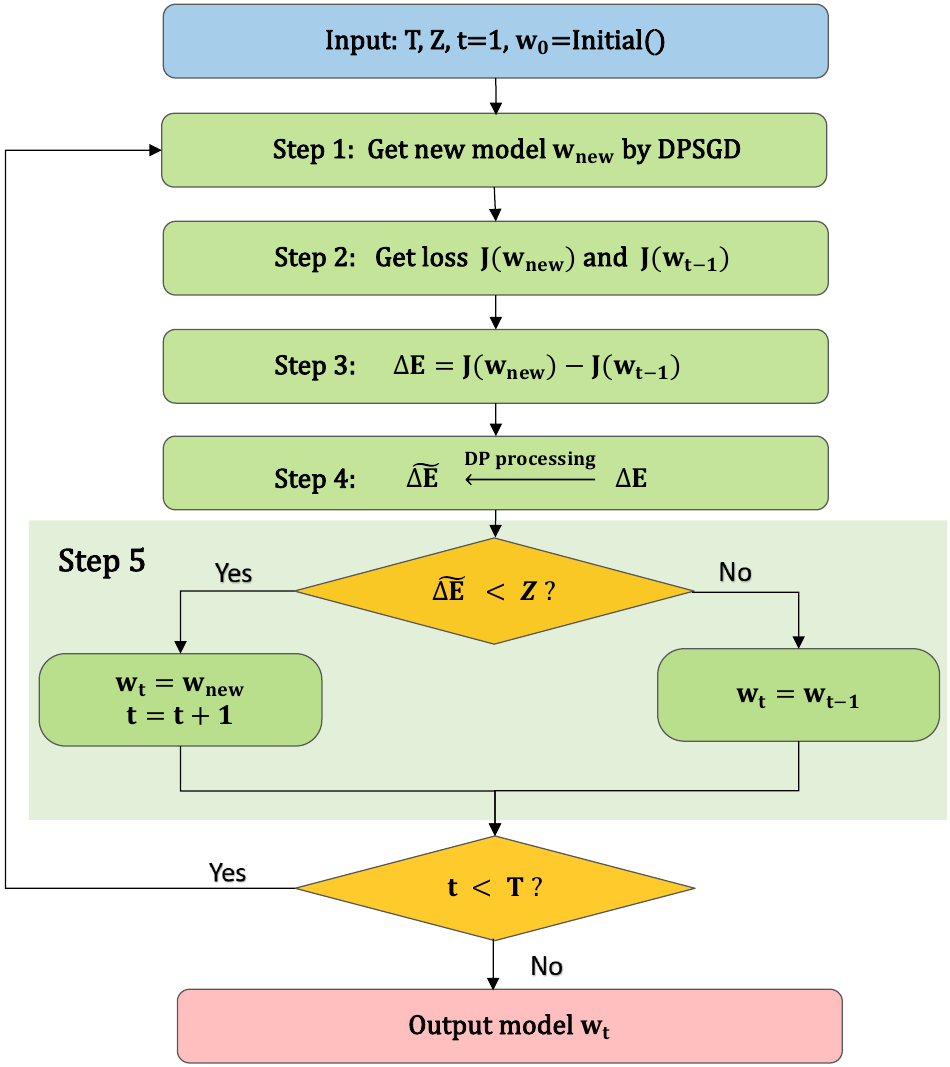}
		\caption{Workflow of DPSUR.}
		\label{figure:flow of DPSUR}
	\end{center}
\end{figure}

Figure~\ref{figure:flow of DPSUR} shows the workflow of DPSUR, which takes as inputs the total number of updates ${T}$, number of updates accepted $t$, acceptance threshold ${Z}$, and initialization model ${w_0}$, executes the following steps, and outputs a final model.

\begin{itemize}
	\item \textbf{Step 1:} In each iteration, we obtain a batch of tuples from the training set via Poisson sampling, and generate an intermediate model $w_{new}$ using the DPSGD algorithm.
	\item \textbf{Step 2:} We evaluate the intermediate model ${w_{new}}$ and ${w_{t-1}}$ on the validation batch $\mathcal{B}_v$ resampled from training set to calculate the loss ${J(w_{new})}$ and ${J(w_{t-1})}$.
	\item \textbf{Step 3:} We compute the difference of loss $\Delta E=J(w_{new})-J(w_{t-1})$ to evaluate the performance of intermediate model.
	\item \textbf{Step 4:} We clip ${\Delta E}$ and add noise to it to satisfy differential privacy, obtaining $\widetilde{\Delta E}$.
	\item \textbf{Step 5:} Given the acceptance threshold ${Z}$, we accept update of the intermediate model $w_{new}$ and $t$ plus 1 if ${\widetilde{\Delta E}}<Z$, or reject it otherwise by reverting back to the last model $w_{t-1}$ that was accepted for update.
\end{itemize}

Note that when $t$ reaches the total number of updates ${T}$, we output the trained model ${w_{t}}$.

Figure~\ref{figure:data division} further shows how we evaluate the model, i.e. obtaining $J(w)$ (step 2 in Figure~\ref{figure:flow of DPSUR}). During each iteration, we first randomly sample a portion of the tuples $\mathcal{B}_t$ from the training set and perform DPSGD training to obtain a model. Then we randomly re-sample a portion of tuples from the training set as validation batch $\mathcal{B}_v$. Finally, the cross-entropy loss function is applied to the $\mathcal{B}_v$ to compute the loss of the current model. It is important to note that sampling from the training set for training and validation serves the purpose of privacy amplification. Specifically, re-sampling from the training set for model validation helps prevent overfitting. The complete DPSUR algorithm will be described in Section~\ref{sec:overall_algorithm}.


\begin{figure}[htb]
	\begin{center}
		\includegraphics[width=0.9\linewidth]{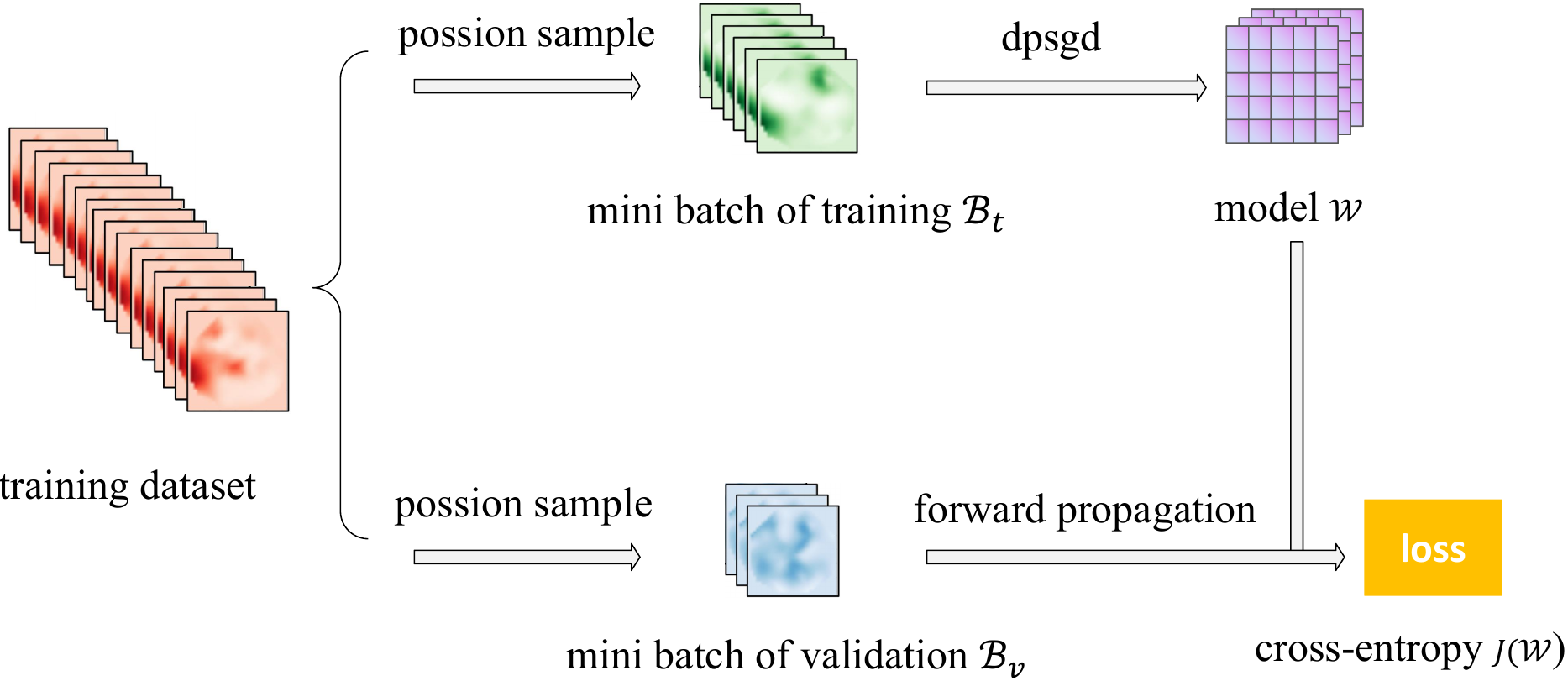}
		\caption{Model evaluation framework in DPSUR.}
		\label{figure:data division}
	\end{center}
\end{figure}

\subsection{Minimal Clipping}
\label{sec:cliping}
As the computation of $\Delta E$ accesses the training set (i.e., a portion of private tuples), it is necessary to perform a differential privacy operation on it. Intuitively, we can clip $\Delta E$ to a certain range $[-C_v,C_v]$, and then add Gaussian noise with mean $0$ and standard deviation $2C_v \cdot \sigma_v $ to ensure differential privacy, as shown in Equation (\ref{equ:loss_dp}). Here, $\sigma_v$ denotes the noise multiplier, and $2C_v$ is the sensitivity of the clipping operation.
\begin{align}\label{equ:loss_dp}
	\begin{split}
		{\widetilde{\Delta E}} & = \textbf{Clip}(\Delta E;C_v) + \sigma_v \cdot 2C_v \cdot \mathcal{N}(0,1),  \\
		&= min(max(\Delta E,-C_v),C_v) +  \sigma_v \cdot 2C_v \cdot \mathcal{N}(0,1)
	\end{split}
\end{align}

However, setting an appropriate clipping bound $C_v$ for $\Delta E$ is challenging. A large $C_v$ helps avoid loss of fidelity to the original values, but it also leads to large injected noise.
A key observation here is that, to assess the quality of the current model, we can simply compare the loss of the model trained in this iteration with that of the previous iteration, i.e., $\Delta E=J(w_{new})-J(w_{t-1})$. If the loss is lower than the previous iteration, i.e., $\Delta E < 0$, the current model $w_{new}$ is better and we accept it, otherwise we reject it. Therefore, we only need to determine whether $\Delta E$ is positive or negative, instead of its absolute value. With that said, we can use a small clipping bound $C_v$ for uniform clipping such that almost all $\Delta E$ are outside the interval $[-C_v, C_v]$, which is illustrated by the red line in Figure~\ref{figure:discrete cilpping}.

\begin{figure}[htb]
	\begin{center}
		\includegraphics[width=0.7\linewidth]{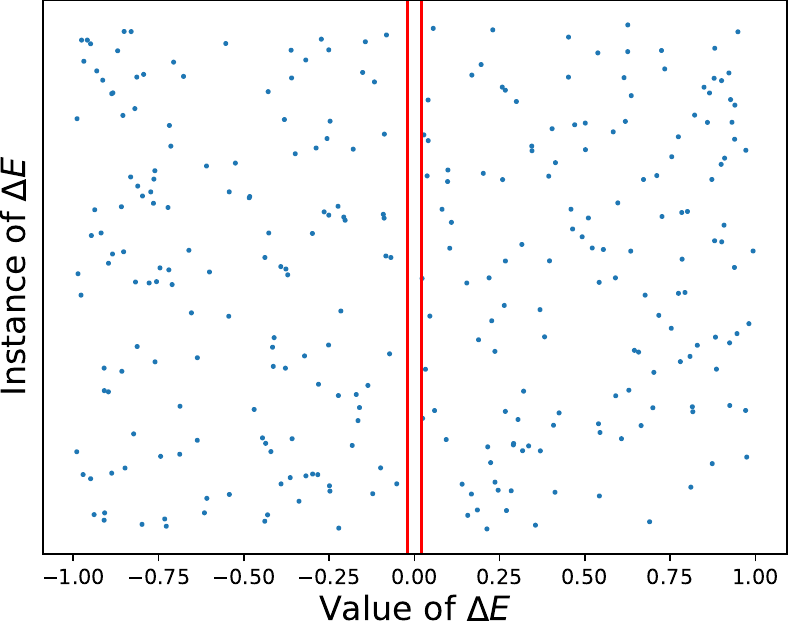}
		\caption{The idea of minimal cilpping.}
		\label{figure:discrete cilpping}
	\end{center}
\end{figure}

\begin{figure*}[htb]
	\centering
	\begin{subfigure}{0.32\linewidth}
		\centering
		\includegraphics[width=0.9\linewidth]{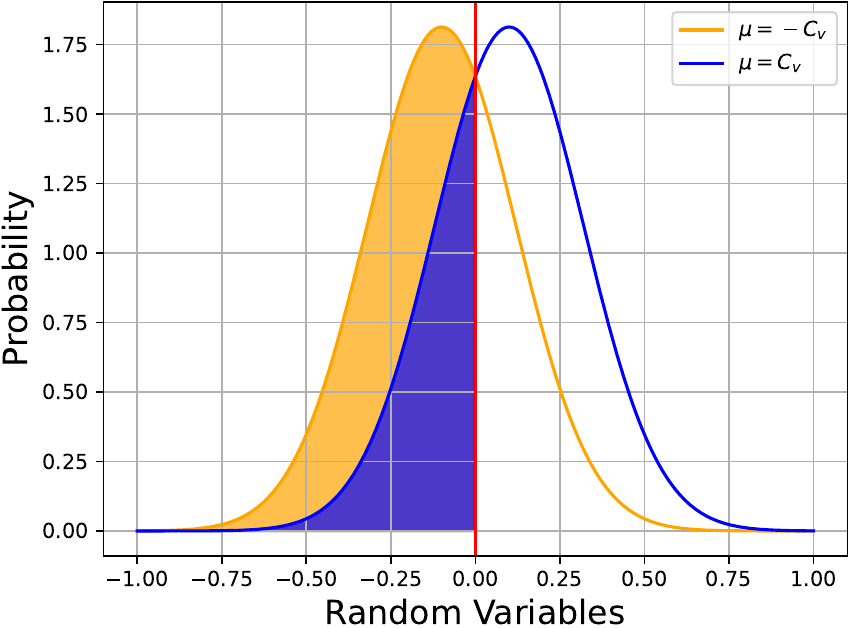}
		\caption{Z = 0}
		\label{figure:gaussian cdf a}
	\end{subfigure}
	\centering
	\begin{subfigure}{0.32\linewidth}
		\centering
		\includegraphics[width=0.9\linewidth]{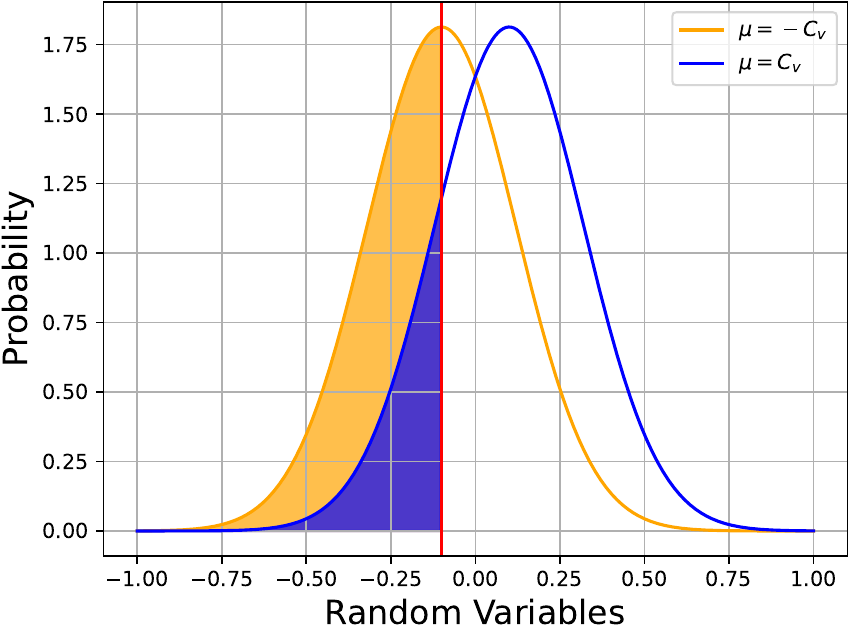}
		\caption{Z = -0.1}
		\label{figure:gaussian cdf b}
	\end{subfigure}
	\centering
	\begin{subfigure}{0.32\linewidth}
		\centering
		\includegraphics[width=0.9\linewidth]{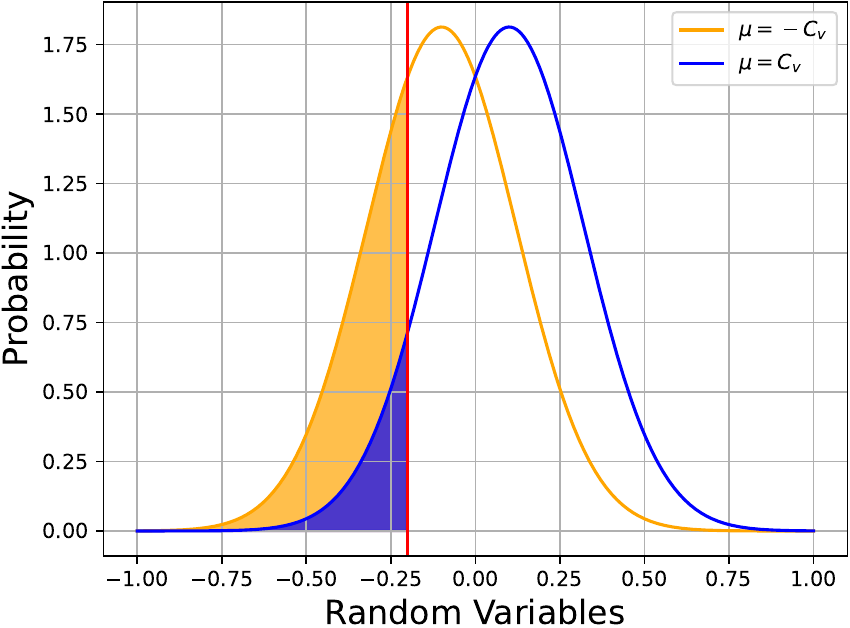}
		\caption{Z = -0.2}
		\label{figure:gaussian cdf c}
	\end{subfigure}
	\caption{The Gaussian probability distributions under different thresholds Z, where the Gaussian distribution $\mathcal{N} (\mu,{(2C_v \cdot \sigma_v)}^2)$ and the $\sigma_v=1, C_v=0.1$}
	\label{figure:gaussian cdf}
\end{figure*}

\begin{equation}\label{equ:Clipping}
	\overline {\Delta E}=\left\{\begin{array}{cc}
		-C_v, & \Delta E \leq -C_v \\
		\Delta E, & -C_v < \Delta E < C_v \\
		C_v, & \Delta E \geq C_v
	\end{array}\right..
\end{equation}

{We clip $\Delta E$ to a certain range $[-C_v, C_v]$ as shown in Equation~\ref{equ:Clipping}. If we choose $C_v$ to be extremely small (e.g., $C_v=1e-10$), we can make almost all $\Delta E$ stay outside $[-C_v,C_v]$, and the clipping process is simplified to Equation~\ref{equ:Minimal Clipping}. Although in the worst case, there could be some value of $\Delta E$ in the interval of $[-C_v, C_v]$, the probability is so small that we can ignore it without affecting our technical analysis. In practice, we can select the clipping bound $C_v$ to be small enough for loss value in gradient descent, e.g., $C_v=0.001$.}

\begin{equation}\label{equ:Minimal Clipping}
	\overline {\Delta E}=\left\{\begin{array}{cc}
		-C_v, & \Delta E < 0 \\
		C_v, & \Delta E > 0
	\end{array}\right..
\end{equation}

So each possible value of $\Delta E$ is discretized as $-C_v$ or $C_v$ after our minimal clipping strategy. A small clipping bound $C_v$ ensures that the injected Gaussian noise is sufficiently small.
Such clipping strategy can quantify the impact of the injected Gaussian noise. As plotted in Figure~\ref{figure:gaussian cdf}, given threshold $Z$ (shown as the red lines), after minimal clipping and adding Gaussian noise, values of $\Delta E$ consist of two (partial) Gaussian distributions with a mean of $C_v$ or $-C_v$ and a standard deviation of $2C_v\sigma_v$, respectively. The blue area represents the probability of the Gaussian distribution with a mean of $C_v$ and less than threshold $Z$, while the orange and blue area together represents the probability of the Gaussian distribution with a mean of $-C_v$ and less than threshold $Z$.

For a real example, let clipping bound $C_v=0.1$ and noise level $\sigma_v=1$. The results by the clipping strategy after adding noise are plotted in Figure~\ref{figure:gaussian cdf a}, where threshold $Z=0$ is used to accept those updates whose $\Delta E<Z$. We observe that the probability of accepting a high-quality model ($\Delta E<0$) is 69.1\%, while the probability of accepting a low-quality model ($\Delta E>0$) is 30.8\%. Next, we slightly move the acceptance threshold $Z$ from $0$ to $-0.2$ to investigate its impact on the acceptance probabilities. Figures~\ref{figure:gaussian cdf b} and \ref{figure:gaussian cdf c} show the results of $Z=-0.1$ and $Z=-0.2$, respectively. We observe that, from $Z=0$ to $Z=-0.1$, the probability of accepting a high-quality model drops much more slowly (from 69.1\% to 50\%) than that of accepting a low-quality model (from 30.8\% to 15.9\%). However, as $Z$ further decreases, the situation is reversed, i.e., the probability of accepting a  high-quality model drops quickly, which goes against the model convergence.
This observation motivates us to explore an optimal threshold $Z$ to maximize the utility gain from model updates.

\subsection{Threshold Mechanism}
\label{sec:threshold}	
In this subsection, we derive an optimal threshold $Z$. Formally, let $Z=\beta \cdot C_v$. Then for the case of $\Delta E <0$, the CDF of the Gaussian distribution $\mathcal{N} (-C_v,{(2C_v \cdot \sigma_v)}^2)$ less than $Z$ is
\begin{align}\label{equ:cdf_-C}
	P(X<Z)=\Phi(\frac{\beta \cdot C_v-(-C_v)}{2 C_v \cdot \sigma_v})=\Phi(\frac{\beta +1 }{2\sigma_v}),
\end{align}
where $X=\mathcal{N} (-C_v,{(2C_v \cdot \sigma_v)}^2)$ and $\Phi(x)=\frac{1}{\sqrt{2 \pi}} \int_{-\infty}^{x} e^{-\frac{t^{2}}{2}} dt$.
Similarly, the CDF of the case $\Delta E >0$ is
\begin{align}\label{equ:cdf_C}
	P(X<Z)=\Phi(\frac{\beta \cdot C_v-C_v}{2 C_v \cdot \sigma_v})=\Phi(\frac{\beta -1 }{2\sigma_v}),
\end{align}
where $X=\mathcal{N} (C_v,{(2C_v \cdot \sigma_v)}^2)$.
According to Equations (\ref{equ:cdf_-C}) and (\ref{equ:cdf_C}), the parameter $C_v$ will be eliminated in the CDF of Gaussian distribution, so the acceptance probability of the algorithm only depends on the parameters $\beta$ and $\sigma_v$.

In Figure~\ref{figure:erf}, we plot the probabilities of accepting high-quality ($\Delta E < 0$) and low-quality ($\Delta E > 0$) models for popular $\sigma_v$: $0.8$, $1.0$, and $1.3$, with respect to $\beta$ ranging from $-3.5$ to $0.5$. Obviously, smaller $\beta$ value results in decreasing probabilities of accepting both high-quality ($\Delta E < 0$) and low-quality ($\Delta E > 0$) models. To find a $\beta$ that maximizes the difference between the two probabilities (i.e., the red dashed vertical line), we find that such $\beta$ seems to be around $-1.0$ for all three $\sigma_v$. 
As will be shown in Figure~\ref{figure:Ablation c} of the experimental results, setting $\beta=-1.0$ does achieve excellent results in all datasets.


\begin{figure}[htb]
	\begin{center}
		\includegraphics[width=0.7\linewidth]{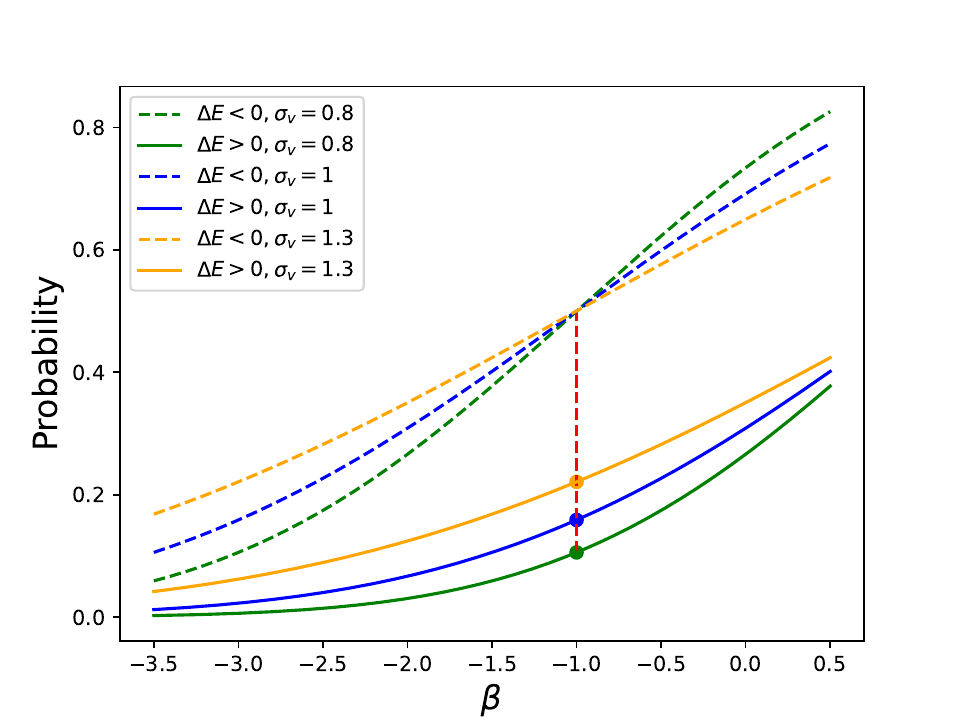}
		\caption{Acceptance probability vs. $\beta$.}
		\label{figure:erf}
	\end{center}
\end{figure}

{\bf Summary.} To address the privacy concerns arising from the evaluation on selective update, we propose a randomized algorithm coupled with a minimal clipping strategy and a threshold mechanism. In particular, by setting a sufficiently small clipping bound $C_v$ (e.g., $0.001$) and the acceptance threshold $Z=\beta \cdot C_v$, the impact of $C_v$ on the selection update is eliminated according to Equations (\ref{equ:cdf_-C}) and (\ref{equ:cdf_C}). Further, a suitable $\beta$ is selected based on probability distributions to achieve maximum utility for selective updates.



%% file: methodology2.tex
\section{Selective Release in DPSUR: An Optimization}
\label{sec:methodology2}

As of now, iteratively calculating the $\widetilde{\Delta E}$ in each iteration requires continual privacy budget consumption. In Section~\ref{sec:selective_release}, we prove that by only releasing the model when a selective update occurs, DPSUR can preserve the privacy budget. Finally, we summarize the overall algorithm of DPSUR in Section~\ref{sec:overall_algorithm}.

\subsection{Gaussian Mechanism with Selective Release}
\label{sec:selective_release}
Recall that when we obtain $\widetilde{\Delta E}$, to protect privacy, we clip $\Delta E$ in $[-C_v,C_v]$ and add Gaussian noise according to Equation~\ref{equ:loss_dp} in {\bf each} iteration. {According to step 5 of Figure~\ref{figure:flow of DPSUR}, we do not update the model until $\widetilde{\Delta E}$ exceeds the threshold $Z$.} If we take one step further and do not release $\widetilde{\Delta E}$ in such cases, as described in Algorithm~\ref{selective publish of Gaussian}, DPSUR only consumes privacy budget when a selective update occurs, i.e., $\widetilde{\Delta E}<\beta \cdot C_v$. In this subsection, we conduct privacy analysis to ensure Algorithm~\ref{selective publish of Gaussian} can satisfy the same RDP guarantee as that of the underlying Gaussian mechanism in Definition~\ref{definition:Gaussian mechanism of RDP}.

{In general, for the query result $f(D)$ on dataset $D$, Algorithm~\ref{selective publish of Gaussian} clip the $f(D)$ in $[0,\mu]$ and adds Gaussian noise until the noisy result falls in a designated interval $[a, b]$.} The following Theorem~\ref{definition:selective publish of Gaussian} shows that when $a \rightarrow -\infty$, Algorithm~\ref{selective publish of Gaussian} can satisfy the same RDP guarantee.

\begin{algorithm}
\caption{Gaussian mechanism with selective release}\label{selective publish of Gaussian}
\KwIn{function $f(\cdot)$, dataset $D$, Gaussian distribution $N$, a given interval $[a,b]$}
\KwOut{$A(D)$ that falls in the interval $[a,b]$ after adding Gaussian noise}
{Clipping $f(D)$ to interval $[0, \mu]$ to obtain sensitivity $\mu$}; \\
$A(D)=f(D)+N(0,\mu^2\sigma^2) $\;
\While{ $ A(D) < a$ or $A(D) > b$}
{$A(D)=f(D)+N(0,\mu^2\sigma^2) $\;
}
\Return $A(D)$
\end{algorithm}

\begin{theorem} \label{definition:selective publish of Gaussian}
Algorithm~\ref{selective publish of Gaussian} satisfies $(\alpha,{\alpha}/{2\sigma^2})-RDP$ when $a \rightarrow -\infty$.
\end{theorem}

\begin{proof}
Figure~\ref{figure: The CDF of Gaussian distribution} plots two normal distributions, namely $N(0,\mu^2\sigma^2)$ and $N(\mu,\mu^2\sigma^2)$, as the output probability distributions of function $f$ on two neighboring datasets whose sensitivity is $\mu$. By selective update and selective release in the threshold range $[a,b]$, the probability distribution outside of this interval will accumulate within the interval. As such, the final output distribution is truncated and transformed into a truncated normal distribution, as shown in Figure~\ref{figure: Truncated normal distribution}. As $x$ is assumed to follow a normal distribution, the truncated normal distribution with mean $0$ and mean $\mu$ can be represented as follows:
\begin{align*}
f(x ; 0, \mu\sigma, a, b)=\left\{\begin{array}{ll}
\frac{1}{\mu\sigma \sqrt{2 \pi}} e^{-\frac{x^{2}}{2 \mu^2\sigma^{2}}} \cdot \frac{1}{\Phi\left(\frac{b}{\mu\sigma}\right)-\Phi\left(\frac{a}{\mu\sigma}\right)} & a \leq x \leq b, \\
0 & \text { otherwise. }
\end{array}\right.
\end{align*}

\begin{align*}
f(x ; \mu, \mu\sigma, a, b)=\left\{\begin{array}{ll}
\frac{1}{\mu\sigma \sqrt{2 \pi}} e^{-\frac{(x-\mu)^{2}}{2 \mu^2\sigma^{2}}} \cdot \frac{1}{\Phi\left(\frac{b-\mu}{\mu\sigma}\right)-\Phi\left(\frac{a-\mu}{\mu\sigma}\right)} & a \leq x \leq b, \\
0 & \text { otherwise. }
\end{array}\right.
\end{align*}

Here $\mu\sigma$ is the standard deviation of the original normal distribution, whereas $a$ and $b$ are the lower and upper truncation values, respectively. $\Phi(x)$ denotes the cumulative distribution function of the standard normal distribution.

Then  we substitute the two truncated normal distributions into Rényi divergence \cite{van2014renyi} to calculate the RDP as follows:
\begin{align}
	\begin{split}
		\begin{aligned}
			D_{\alpha}( & f(x ; 0, \mu\sigma, a, b )||f(x ; \mu, \mu\sigma, a, b)) \\
			= & \frac{1}{\alpha-1} \cdot \ln \int_{a}^{b} \frac{[f(x ; 0, \mu\sigma, a, b )]^\alpha}{[f(x ; \mu, \mu\sigma, a, b)]^{\alpha-1}} \mathrm{d} x \\
			= & \frac{1}{\alpha-1}\cdot \ln \{\frac{(\Phi(\frac{b-\mu}{\mu\sigma })-\Phi(\frac{a-\mu}{\mu\sigma } ))^{\alpha-1}  }{(\Phi(\frac{b}{\mu\sigma })-\Phi(\frac{a}{\mu\sigma }))^{\alpha}} \cdot \frac{1}{\mu\sigma \sqrt{2 \pi}} \int_{a}^{b} \exp[(-x^{2} \\
			&  + 2(1-\alpha) \mu x-(1-\alpha) \mu^{2}) /(2 \mu^2\sigma^2)] \mathrm{d} x \} \\
   		 = & \frac{1}{\alpha-1} \cdot \{\frac{\alpha(\alpha-1)}{2\sigma ^2} + \ln [ \frac{(\Phi(\frac{b-\mu}{\mu\sigma })-\Phi(\frac{a-\mu}{\mu\sigma } ))^{\alpha-1}  }{(\Phi(\frac{b}{\mu\sigma })-\Phi(\frac{a}{\mu\sigma }))^{\alpha}} \\
	   	& \cdot ({\Phi(\frac{b-(1-\alpha)\mu}{\mu\sigma})-\Phi(\frac{a-(1-\alpha)\mu}{\mu\sigma})})]\} \\
   		= & \frac{\alpha }{2 \sigma^2} + \frac{1}{\alpha-1} \cdot \ln\{\frac{(\Phi(\frac{b-\mu}{\mu\sigma })-\Phi(\frac{a-\mu}{\mu\sigma } ))^{\alpha-1}  }{(\Phi(\frac{b}{\mu\sigma })-\Phi(\frac{a}{\mu\sigma }))^{\alpha}} \cdot [\Phi(\frac{b-(1-\alpha)\mu}{\mu\sigma})\\
			&  -\Phi(\frac{a-(1-\alpha)\mu}{\mu\sigma})]\} ,
		\end{aligned}
	\end{split}
\end{align}

where $\Phi(x)=\frac{1}{\sqrt{2 \pi}} \int_{-\infty}^{x} e^{-\frac{t^{2}}{2}} dt \nonumber$.

As $a \rightarrow -\infty$, we can get:
\begin{align}
	\begin{split}
		\begin{aligned}
        D_{\alpha}( & f(x ;0, \mu\sigma, a, b )||f(x ; \mu, \mu\sigma, a, b)) \\
		= & \frac{\alpha }{2 \sigma^2} +\frac{1}{\alpha-1} \cdot \ln[\frac{(\Phi(\frac{b-\mu}{\mu\sigma }))^{\alpha-1}  }{(\Phi(\frac{b}{\mu\sigma }))^{\alpha}} \cdot \Phi(\frac{b+(\alpha-1)\mu}{\mu\sigma})]\\
		\end{aligned}
	\end{split}
\end{align}

\begin{figure}
  \centering
  \begin{subfigure}{0.49\linewidth}
    \centering
    \includegraphics[width=1.0\linewidth]{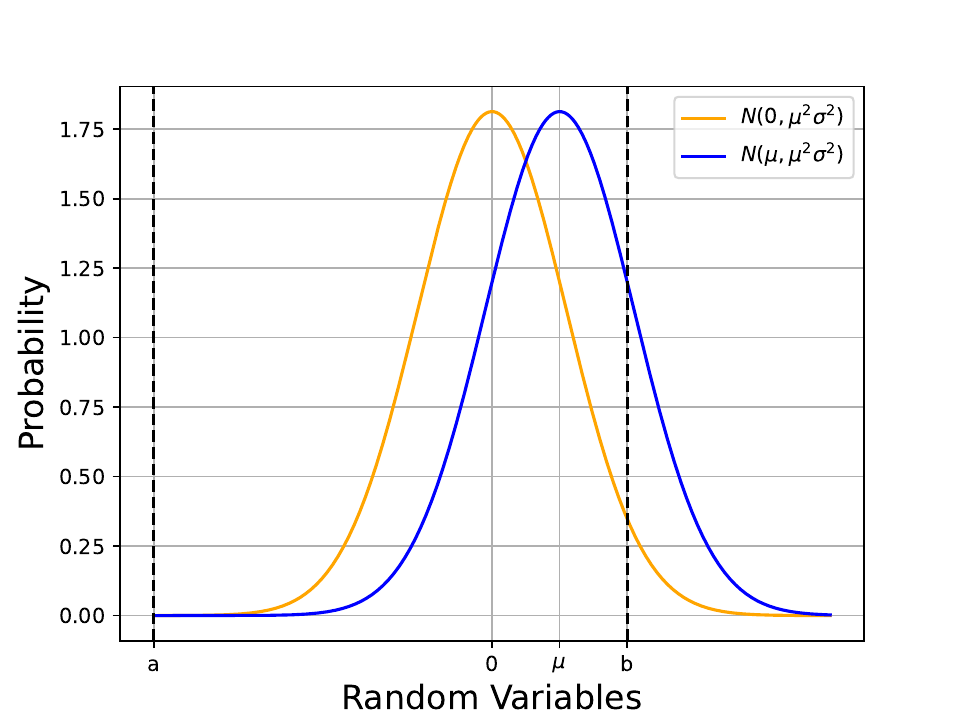}
    \caption{Before truncation.}
    \label{figure: The CDF of Gaussian distribution}
  \end{subfigure}
  \hfill
  \begin{subfigure}{0.49\linewidth}
    \centering
    \includegraphics[width=1.0\linewidth]{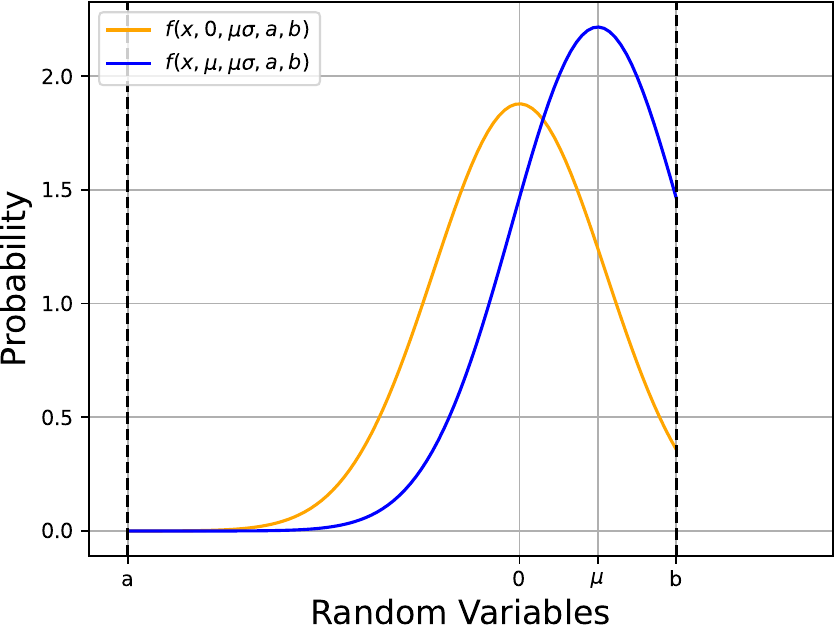}
    \caption{After truncation.}
    \label{figure: Truncated normal distribution}
  \end{subfigure}
  \caption{Before and after truncating the normal distribution.}
  \label{figure: Before and after truncating the normal distribution}
\end{figure}

Similarly, we can get :
\begin{align}
	\begin{split}
		\begin{aligned}
        D_{\alpha}( & f(x ;\mu, \mu\sigma, a, b )||f(x ; 0, \mu\sigma, a, b)) \\
        = & \frac{1}{\alpha-1} \cdot \ln \int_{a}^{b} \frac{[f(x ; \mu, \mu\sigma, a, b )]^\alpha}{[f(x ; 0, \mu\sigma, a, b)]^{\alpha-1}} \mathrm{d} x \\
        = & \frac{1}{\alpha-1} \cdot \{\frac{\alpha(\alpha-1)}{2\sigma ^2} + \ln [ \frac{(\Phi(\frac{b}{\mu\sigma })-\Phi(\frac{a}{\mu\sigma } ))^{\alpha-1}  }{(\Phi(\frac{b-\mu}{\mu\sigma })-\Phi(\frac{a-\mu}{\mu\sigma }))^{\alpha}} \\
        & \cdot ({\Phi(\frac{b-\alpha \mu}{\mu\sigma})-\Phi(\frac{a-\alpha \mu}{\mu\sigma})})]\} \\
		= & \frac{\alpha }{2 \sigma^2} +\frac{1}{\alpha-1} \cdot \ln [\frac{(\Phi(\frac{b}{\mu\sigma }))^{\alpha-1}  }{(\Phi(\frac{b-\mu}{\mu\sigma }))^{\alpha}} \cdot \Phi(\frac{b-\alpha\mu}{\mu\sigma})] \\
		\end{aligned}
	\end{split}
\end{align}


According Theorem~\ref{definition:CDF of Gaussian}, we have:
\[
	D_{\alpha}(f(x ; 0, \mu\sigma, a, b )||f(x ; \mu, \mu\sigma, a, b)) \leq \alpha /{2 \sigma^{2}}, and
\]
\[
	D_{\alpha}(f(x ; \mu, \mu\sigma, a, b )||f(x ; 0, \mu\sigma, a, b)) \leq \alpha /{2 \sigma^{2}}.
\]

Therefore, Theorem~\ref{definition:selective publish of Gaussian} is proved. The details of RDP of two truncated normal distributions are presented in Appendix~\ref{Appendix:Rényi Divergence of Truncated Gaussian distribution}.
\end{proof}

Next, we will prove Theorem~\ref{definition:CDF of Gaussian}. This theorem is an inequality proof of the cumulative distribution function (CDF) of the normal distribution, and it provides the foundation for the proof of Theorem~\ref{definition:selective publish of Gaussian}.

\begin{theorem} \label{definition:CDF of Gaussian}
If \; $A= (\Phi(\frac{b-\mu}{\mu\sigma }))^{\alpha-1} \cdot \Phi(\frac{b+(\alpha-1)\mu}{\mu\sigma})/ (\Phi(\frac{b}{\mu\sigma }))^{\alpha}, B= (\Phi(\frac{b}{\mu\sigma }))^{\alpha-1} \cdot \Phi(\frac{b-\alpha\mu}{\mu\sigma})/(\Phi(\frac{b-\mu}{\mu\sigma }))^{\alpha}$,\; where\; $\Phi(x)=\frac{1}{\sqrt{2 \pi}} \int_{-\infty}^{x} e^{-\frac{t^{2}}{2}} dt$. $A, B \leq 1$,\; when \;$\mu>0,\sigma > 0$\; and\; $\alpha>1$.
\end{theorem}
\begin{proof}

The second derivative of $\ln{[\Phi(x)]}$ is:
\begin{align}
\begin{split}
\begin{aligned}
         \ln^{\prime\prime}{[\Phi(x)}]
          &= \frac{\Phi^{\prime\prime}(x) \cdot \Phi(x)-\Phi^{\prime}(x) \cdot \Phi^{\prime}(x)}{[\Phi(x)]^2} \\
          &= \frac{-x \cdot \Phi^{\prime}(x) \cdot \Phi(x)-\Phi^{\prime}(x) \cdot \Phi^{\prime}(x)}{[\Phi(x)]^2} ,\\
\end{aligned}
\end{split}
\end{align}

where $\Phi^{\prime}(x)=\frac{1}{\sqrt{2 \pi}}  e^{-\frac{x^{2}}{2}} \nonumber$.

When $x \ge 0$, since $\Phi(x), \Phi^{\prime}(x) >0$, $\ln^{\prime\prime}{[\Phi(x)}] < 0$ always hold.

When $x < 0$, $\Phi^{\prime}(x)>0$. As such, proving $-x \cdot \Phi^{\prime}(x) \cdot \Phi(x)-\Phi^{\prime}(x) \cdot \Phi^{\prime}(x) <0 $ \; is equivalent to proving \; $-x\cdot \Phi(x)- \Phi^{\prime}(x)<0$. We let $K(x)=-x\cdot \Phi(x)- \Phi^{\prime}(x)$, so the derivative for $K(x)$ is:
\begin{align}
\begin{split}
\begin{aligned}
         K^{\prime}(x)&= -\Phi(x)+(-x) \cdot \Phi^{\prime}(x) - \Phi^{\prime\prime}(x) \\
         &= -\Phi(x)+ (-x) \cdot \Phi^{\prime}(x) - (-x) \cdot \Phi^{\prime}(x) \\
          &=  -\Phi(x) < 0
\end{aligned}
\end{split}
\end{align}

Since $\textstyle \lim_{x \to -\infty}K(x)=0$, $K(x)<0$ when $x<0$. In summary, $\ln^{\prime\prime}{[\Phi(x)}] < 0$ always holds. In other words, $\ln^{\prime\prime}{[\Phi(x)}]$ is the concave function when $x \in (-\infty,\infty)$. According to the properties of the concave function, we can obtain:
\begin{align}
\begin{split}
\begin{aligned}
         & \lambda \cdot \ln{[\Phi(x_1)]}+ (1-\lambda) \ln{[\Phi(x_2)]} \leq \ln{[\Phi(\lambda x_1+(1-\lambda)x_2)]} ,\\
         & \textrm{where} \; \lambda \in (0,1) \; \textrm{and} \; x_1,x_2 \in (-\infty,\infty) \nonumber.
\end{aligned}
\end{split}
\end{align}

Let $\lambda=\frac{\alpha-1}{\alpha}(\alpha>1)$\; ,\; $x_1=\frac{b-\mu}{\mu\sigma},\; x_2=\frac{b+(\alpha-1)\mu}{\mu\sigma}$ \;or\; $x_1=\frac{b}{\mu\sigma},\; x_2=\frac{b-\alpha\mu}{\mu\sigma}$. Then the following two inequalities hold:

\begin{align}
\begin{split}
\begin{aligned}
         & \frac{(\alpha-1)}{\alpha} \cdot \ln{[\Phi(\frac{b-\mu}{\mu\sigma})]}+\frac{1}{\alpha} \ln{[\Phi(\frac{b+(\alpha-1)\mu}{\mu\sigma})]} \leq \ln{[\Phi(\frac{b}{\mu\sigma})]} \\
        & \frac{(\alpha-1)}{\alpha} \cdot \ln{[\Phi(\frac{b}{\mu\sigma})]}+\frac{1}{\alpha} \ln{[\Phi(\frac{b-\alpha\mu}{\mu\sigma})]} \leq \ln{[\Phi(\frac{b-\mu}{\mu\sigma})]} \\
\end{aligned}
\end{split}
\end{align}

In the end, we can obtain the following:
\begin{align}
\begin{split}
\begin{aligned}
         & (\Phi(\frac{b-\mu}{\mu\sigma }))^{\alpha-1} \cdot \Phi(\frac{b+(\alpha-1)\mu}{\mu\sigma}) \leq (\Phi(\frac{b}{\mu\sigma }))^{\alpha}\\
        & (\Phi(\frac{b}{\mu\sigma }))^{\alpha-1} \cdot \Phi(\frac{b-\alpha\mu}{\mu\sigma}) \leq (\Phi(\frac{b-\mu}{\mu\sigma }))^{\alpha}, \\
\end{aligned}
\end{split}
\end{align}

where $\mu,\sigma >0$ and $\alpha>1 \nonumber$.

Theorem~\ref{definition:CDF of Gaussian} is proved.
\end{proof}

This proves Algorithm~\ref{selective publish of Gaussian} satisfies the same $(\alpha,{\alpha}/{2\sigma^2})$-RDP as Gaussian mechanism of RDP~\cite{mironov2017renyi}. In other words, DPSUR only consumes privacy budget when the update is selected and released based on the interval. As such, in Figure~\ref{figure:flow of DPSUR}, the privacy budget of computing $\widetilde{\Delta E}$ is consumed only if $\widetilde{\Delta E}<Z$.

\subsection{DPSUR: Putting Things Together}
\label{sec:overall_algorithm}
Now we describe the overall algorithm of DPSUR in Algorithm~\ref{algorithm1}, which consists of the following two steps.

\textbf{i. DPSGD (Lines 3-8).} This part is the traditional DPSGD procedure. First, a small batch of samples $\mathcal{B}_t$ are randomly selected from the training datasets (Line~\ref{line:3}). For each sample $x_i \in \mathcal{B}_t$, its gradient values are calculated and clipped so that the $l_2$ norm of the gradients is not greater than the clipping bound $C_t$ (Lines~\ref{line:4}-\ref{line:6}). In Line~\ref{line:7}, the clipped gradients are first summed up, and then added Gaussian noise $\mathcal{N}(0,C_t^2\sigma_t^2)$ to satisfy differential privacy, and finally averaged. As such, the sensitivity is $C_t$ here. Gradient descent is then performed using these noisy gradients to obtain a new temporary model $w_{new}$ for the current iteration (Line~\ref{line:8}).

\textbf{ii. Selective update (Lines 9-18).}
First, a small batch of samples $\mathcal{B}_v$ are randomly selected from the training set (Line~\ref{line:9}). Then we calculate the loss for temporary model $J(w_{new})$ and the latest accepted model $J(w_{t-1})$, where $J(w)=1 /|\mathcal{B}_v| \sum_{x \in \mathcal{B}_v} \mathcal{L}(w, x)$, and subtract them to get $\Delta E$ (Lines~\ref{line:10}-~\ref{line:11}). To get the sensitivity, $\Delta E$ is clipped to [$-C_v,C_v$], which means that one less or one more sample produces a maximum variation of $2C_v$ for $\Delta E$ (Line~\ref{line:12}). To ensure differential privacy, Gaussian noise $\mathcal{N}(0,{\sigma_{v}}^2\cdot{(2 C_v)}^2)$ is added to obtain a noisy version $\widetilde {\triangle E}$ (Lines~\ref{line:13}). The temporary model $w_{new}$ will be accepted and $t$ plus 1 if $\widetilde {\triangle E} < \beta \cdot C_v$, where $\beta$ is the acceptance threshold parameter (Lines~\ref{line:15}-~\ref{line:16}). Otherwise, the temporary model $w_{new}$ is rejected, and the model from the last iteration is returned (Line~\ref{line:18}).

The above two steps are repeated until the entire privacy budget is consumed. We will provide rigorous privacy analysis in the next section.

\begin{algorithm}
\caption{Overall algorithm of DPSUR}\label{algorithm1}
\KwIn{training datasets $\{x_1, x_2, \ldots,x_N\}$, loss function $\mathcal{L}(\theta,x)$. Parameters: learning rate $\eta$, batch size for training $B_t$, noise multiplier for training $\sigma_t$, clipping bound for training $C_t$, batch size for validation $B_v$, noise multiplier for validation $\sigma_v$, clipping bound for validation $C_v$, threshold parameter $\beta$ }
\KwOut{the final trained model $w_t$}
Initialize $t=1$, $w_0=\text{Initial()}$\;
\While{$t<T$}
{
Randomly sample a batch $\mathcal{B}_t$ with probability $\frac{B_t}{N}$\; \label{line:3}
\For{$x_i\in \mathcal{B}_t$ \label{line:4}}
{
Compute $g_t(x_i)\leftarrow \nabla \mathcal{L}_{(w_t,x_i)}$\; \label{line:5}
$\overline{g}_t(x_i)\leftarrow g_t{(x_i)}/\max(1,\frac{||g_t{(x_i)}||2}{C_t})$\; \label{line:6}
}
$\widetilde{g}_t\leftarrow\frac{1}{|\mathcal{B}_t|}(\sum_{i \in {\mathcal{B}_t}}
\overline{g}_t(x_i)+\mathcal{N}(0,\sigma^2 {C_t}^2))$\; \label{line:7}
$w_{new}=w_{t-1}-\eta_t\widetilde{g}_t$\; \label{line:8}
Poisson sampling a batch $\mathcal{B}_v$ with probability $\frac{B_v}{N}$\; \label{line:9}
Compute loss $J(w_{new})$ and $J(w_{t-1})$ by batch $\mathcal{B}_v$  \; \label{line:10}
$\triangle E=J(w_{new})-J(w_{t-1})$\; \label{line:11}
$ \overline {\triangle E}=\min(\max(\triangle E,-C_v),C_v) $\;\label{line:12}
$\widetilde {\triangle E}=\overline {\triangle E} +\mathcal{N}(0,{\sigma_{v}}^2\cdot{(2 C_v)}^2)$\; \label{line:13}
\eIf{$ \widetilde {\triangle E} < \beta \cdot C_v $}
{$w_{t}=w_{new}$\; \label{line:15}
$t=t+1$; \label{line:16}}
{$w_{t}=w_{t-1}$; \label{line:18}}
}

\Return $w_{t}$;
\end{algorithm}

%% file: analysis.tex
\section{Privacy Analysis} \label{sec:Analysis}
This section establishes the privacy guarantee of DPSUR. Since DPSUR is non-interactive, it only releases the final model accumulated from all accepted model updates. In the following, we first analyze the privacy loss of each accepted model update, and then derive the total privacy loss based on sequential composition.


For each accepted model update in Algorithm~\ref{algorithm1}, there are only two places where the raw training data is accessed. One is the computation of model updates $w_{new}$ in the training phase, and the other is the computation of test values $\widetilde{\Delta E}$ in the validation phase. As shown in Figure~\ref{figure: The privacy analysis of DPSUR}, the model update and test value sequences are in the form of ``rejected, ..., rejected, accepted''. 

In the validation phase, the computation of each test value $\widetilde{\Delta E}$ takes as input the training set and the corresponding model update $w_{new}$, and the latter is a function of the training set. Due to the function composition, the computation of $\widetilde{\Delta E}$ can be regarded as a function of the training set whose output is a value in $[-C_v, C_v]$. Therefore, Line~\ref{line:13} in Algorithm~\ref{algorithm1} ensures that the computation satisfies differential privacy due to the Gaussian mechanism. Furthermore, computing and testing the values $\widetilde{\Delta E}$ until one value is accepted satisfies the same differential privacy according to Theorem~\ref{definition:selective publish of Gaussian}.


In the training phase, since the accepted model update is selected solely based on the test value $\widetilde{\Delta E}$ which satisfies differential privacy, the model update selection satisfies differential privacy due to the post-processing property.
Additionally, all rejected model updates are never released and thus consume no privacy budget, even though they have been computed. Therefore, the privacy loss in the training phase comes only from the computation of accepted model updates.

Based on the above analysis, we will use the Rényi Differential Privacy (RDP) approach to calculate the privacy loss in the training and validation phases in Sections~\ref{sec:privacy_analysis_training} and~\ref{sec:privacy_analysis_validation} respectively, compose them sequentially, and finally convert the RDP into $(\epsilon,\delta)$-DP in Section~\ref{sec:overall_privacy_analysis}.

\begin{figure}[htb]
	\begin{center}
		\includegraphics[width=0.9\linewidth]{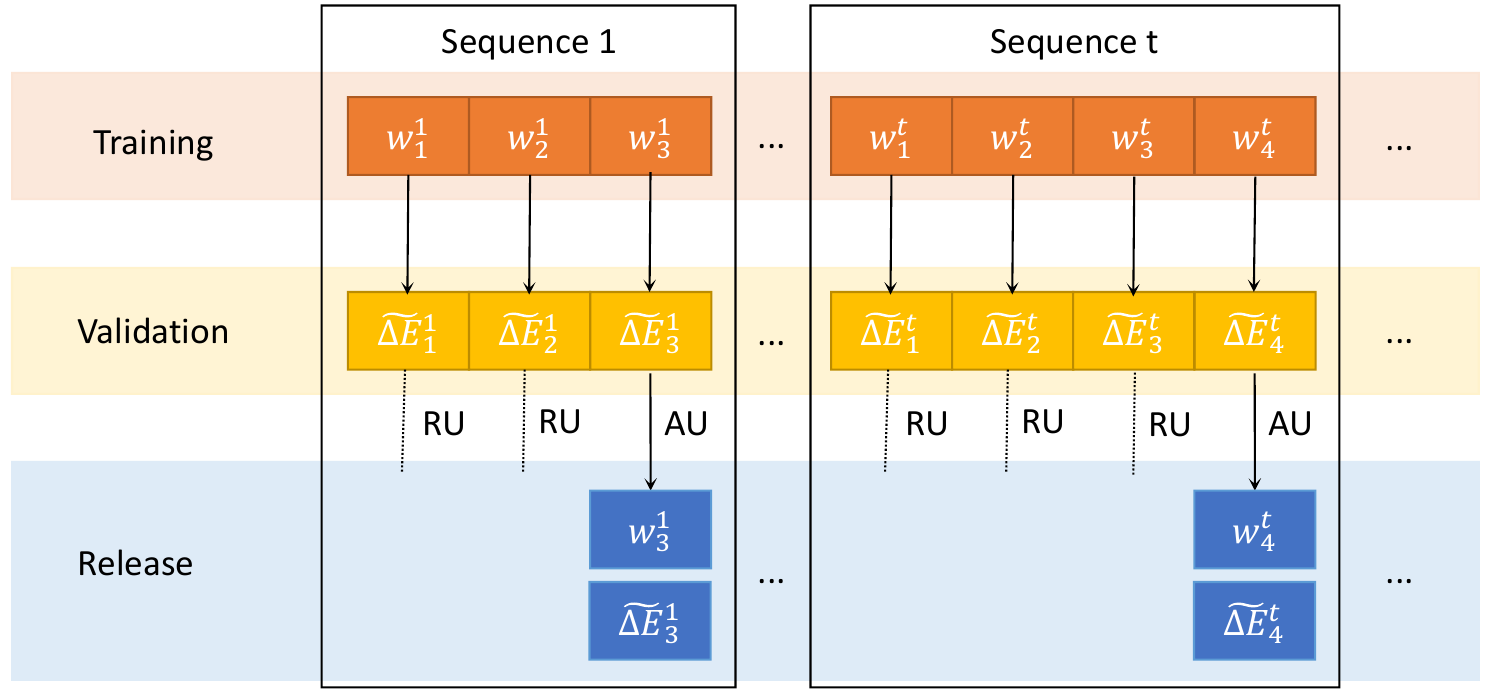}
		\caption{The privacy analysis of DPSUR. (RU: Rejected Update, AU: Accepted Update)}
		\label{figure: The privacy analysis of DPSUR}
	\end{center}
\end{figure}

\vspace{-5mm}
\subsection{Privacy Analysis of Training}
\label{sec:privacy_analysis_training}
As mentioned above, the RDP of the training phase is only resulted from the accepted model updates. Theorem~\ref{the:rdp-of-dpsgd} gives the proof of RDP in the training phase.


\begin{theorem}\label{the:rdp-of-dpsgd} After accepting $t$ model updates, the RDP of the training phase of DPSUR satisfies:
	\begin{equation}\label{equ:epsilon}
		R_{train}(\alpha)= \frac{t}{\alpha-1}\ln \left[\sum_{i=0}^{\alpha}\left(\begin{array}{l}
			\alpha \\ i
		\end{array}\right)(1-q)^{\alpha-i} q^{i} \exp \left(\frac{i^{2}-i}{2 \sigma_t^{2}}\right)\right],
	\end{equation}
	where $q=\frac{B_t}{N}$, $\sigma_t$ is noise multiplier of the training phase, and $\alpha > 1$ is the order.
\end{theorem}

\begin{proof}
	We will prove this in the following two steps: (i) use the RDP of the sampling Gaussian mechanism to calculate the privacy cost of each accepted model update, which relies on Definitions~\ref{def:sgm} and \ref{def:rdp}, and (ii) use the composition of RDP mechanisms to compute the privacy cost of multiple accepted model updates by Lemma~\ref{lem:composition}.
	

	Definitions~\ref{def:sgm} and \ref{def:rdp} define Sampled Gaussian Mechanism (SGM) and its Rényi Differential Privacy (RDP), respectively.
	\begin{definition}\label{def:sgm}
		({\bf Sampled Gaussian Mechanism (SGM)~\cite{IlyaMironov2019RnyiDP}}). Let $f$ be a function mapping subsets of $S$ to $\mathbb{R}^d$. We define the Sampled Gaussian Mechanism (SGM) parameterized with the sampling rate $0 < q \leq 1$ and the  $\sigma > 0$ as
		\begin{equation}
			\begin{aligned}
				S G_{q, \sigma}(S) \triangleq & f(\{x: x \in S \text { is sampled with probability } q\}) \\
				&+\mathcal{N}\left(0, \sigma^{2} \mathbb{I}^{d}\right).
			\end{aligned}
		\end{equation}
		In DPSUR, $f$ is the clipped gradient evaluation on sampled data points $f(\{x_i\}_{i\in B}) = \sum_{i\in B} \overline{g}_t(x_i)$. If $ \overline{g}_t$ is obtained by clipping $g_t$ with a gradient norm bound $C$, then the sensitivity of $f$ is $C$.
	\end{definition}
	
	\begin{definition}\label{def:rdp}
		({\bf RDP privacy budget of SGM~\cite{IlyaMironov2019RnyiDP}}). Let $SG_{q,\sigma}$, be the Sampled Gaussian Mechanism for some function $f$. If $f$ has a sensitivity of $1$, $SG_{q,\sigma}$ satisfies $(\alpha,R)$-RDP if
		\begin{equation}
			R \leq \frac{1}{\alpha-1} \ln [max(A_{\alpha}(q,\sigma),B_{\alpha}(q,\sigma))],
		\end{equation}
		where
		\begin{equation}
			\left\{\begin{array}{l}
				A_{\alpha}(q, \sigma) \triangleq \mathbb{E}_{z \sim \mu_{0}}\left[\left(\mu(z) / \mu_{0}(z)\right)^{\alpha}\right] \\
				B_{\alpha}(q, \sigma) \triangleq \mathbb{E}_{z \sim \mu}\left[\left(\mu_{0}(z) / \mu(z)\right)^{\alpha}\right]
			\end{array}\right.
		\end{equation}
		with $\mu_{0} \triangleq \mathcal{N}\left(0, \sigma^{2}\right), \mu_{1} \triangleq \mathcal{N}\left(1, \sigma^{2}\right) \mbox { and } \mu \triangleq(1-q) \mu_{0}+q \mu_{1}$.
		
		Furthermore, it holds for $\forall(q,\sigma)\in(0,1] \times \mathbb{R}^{+},A_{\alpha}(q,\sigma) \geq B_{\alpha}(q, \sigma) $. Thus, $ S G_{q, \sigma}$  satisfies  $\left(\alpha, \frac{1}{\alpha-1} \ln \left(A_{\alpha}(q, \sigma)\right)\right)$-RDP.
		
		Finally, the existing work \cite{IlyaMironov2019RnyiDP} describes a procedure to compute $A_{\alpha}(q,\sigma)$ depending on integer $\alpha$ as Eq.~\eqref{equ:a-alpha}.
		\begin{equation}\label{equ:a-alpha}
			A_{\alpha}=\sum_{k=0}^{\alpha}\left(\begin{array}{l}
				\alpha \\ k
			\end{array}\right)(1-q)^{\alpha-k} q^{k} \exp \left(\frac{k^{2}-k}{2 \sigma^{2}}\right)
		\end{equation}
	\end{definition}
	
	Lemma~\ref{lem:composition} shows the composition property of RDP mechanisms.
	\begin{lemma}\label{lem:composition}
		({\bf Composition of RDP~\cite{mironov2017renyi}}). For two randomized mechanisms $f, g$ such that $f$ is $(\alpha,R_1)$-RDP and $g$ is $(\alpha,R_2)$-RDP the composition of $f$ and $g$ which is defined as $(X, Y )$(a sequence of results), where $ X \sim f $ and $Y \sim g$, satisfies $(\alpha,R_1+R_2)-RDP$
	\end{lemma}\label{lem:iterations}
	
	According to Definitions~\ref{def:sgm} and \ref{def:rdp}, and Lemma~\ref{lem:composition}, Theorem~\ref{the:rdp-of-dpsgd} is proved.

\end{proof}





\subsection{Privacy Analysis of Validation}
\label{sec:privacy_analysis_validation}
As Section~\ref{sec:selective_release} mentioned above, the RDP of validation will be accumulated only when $\widetilde{\Delta E} < \beta \cdot C_v$.

\begin{theorem}\label{the:rdp-of-validation} After accepting $t$ tests of $\tilde{\Delta E}$, the RDP of the validation phase satisfies:
	\begin{equation}\label{equ:epsilon}
		R_{valid}(\alpha)= \frac{t}{\alpha-1}\ln \left[\sum_{i=0}^{\alpha}\left(\begin{array}{l}
			\alpha \\ i
		\end{array}\right)(1-q)^{\alpha-i} q^{i} \exp \left(\frac{i^{2}-i}{2 \sigma_v^{2}}\right)\right],
	\end{equation}
	where $q=\frac{B_v}{N}$, $\sigma_v$ is noise multiplier of the validation phase, and $\alpha > 1$ is the order.
\end{theorem}

The proof is similar to that of the training phase, so we omit it.

\subsection{Overall Privacy Analysis of DPSUR}
\label{sec:overall_privacy_analysis}
Since both training and validation phases access the same training set, we need to combine their RDPs sequentially using Lemma~\ref{lem:composition}, and then use Lemma~\ref{lem:conversion} to convert it to $(\epsilon,\delta)$-DP. Therefore, the final privacy loss of DPSUR is as follows:
\begin{theorem}\label{the:privacy-loss-DPSUR}
	({\bf Privacy loss of DPSUR}). The privacy loss of DPSUR satisfies:
		\begin{align}
		\begin{split}
			\begin{aligned}
				(\epsilon,\delta)=&(R_{train}(\alpha)+ R_{valid}(\alpha) +\ln ((\alpha-1) / \alpha) \\
				&-(\ln \delta+ \ln \alpha) /(\alpha-1),\delta),
			\end{aligned}
		\end{split}
	\end{align}
where $0<\delta<1$, $R_{train}(\alpha)$ is the RDP of training which is computed by Theorem~\ref{the:rdp-of-dpsgd}, and $R_{valid}(\alpha)$ is the RDP of validation which is computed by Theorem~\ref{the:rdp-of-validation}.
\end{theorem}

\subsection{Discussion of Privacy}\label{sec:Discussion of Privacy}
Our privacy analysis shows that DPSUR strictly adheres to the principles of differential privacy, limiting adversaries to conduct differential attacks solely based on the algorithm's output. However, it is worth noting that DPSUR may be susceptible to interactive side-channel attacks. For instance, a strong adversary (e.g., the hypervisor of a guest OS in the cloud) who has access to DPSUR's internal update/release status could measure the time interval of two adjacent model updates to infer the number of rejections in between and thus cause privacy breaches. To mitigate such threats, we suggest introducing random waiting time for update acceptance cases. Nonetheless, we emphasize that DPSGD is supposed to work in a non-interactive training scenario where the attacker can only access the final output model. Since most real-world machine learning systems are non-interactive during training, the privacy guarantee provided by DPSUR remains sufficient and consistent with other DPSGD variants. 

To further validate such privacy guarantee in real-world scenarios, in Section~\ref{sec:Resilience Against Member Inference Attacks} we conduct membership inference attacks on the trained models. The experimental results indicate that DPSUR exhibits strong defense against membership inference attacks, thus safeguarding the privacy of training data.

%% file: results.tex
\section{Experimental Evaluation}
\label{sec:eval}
In this section, we conduct experiments to demonstrate the performance of DPSUR over four real datasets and popular machine learning models. And we perform experiments involving two member inference attacks to show the privacy-preserving effect of DPSUR. All experiments are implemented in Python using PyTorch \cite{pytorch2018pytorch}. Codes to reproduce our experiments are available at \href{https://github.com/JeffffffFu/DPSUR}{https://github.com/JeffffffFu/DPSUR}.

\subsection{Experimental Setting}
\subsubsection{Baseline}
{We compare DPSUR with DPSGD~\cite{abadi2016deep} and four state-of-the-art variants, namely DPSGD with important sampling~\cite{wei2022dpis}, handcrafted features~\cite{tramer2020differentially}, tempered sigmoid activation~\cite{papernot2021tempered}, {and adaptive learning rate~\cite{lee2018concentrated}}, which we refer to as DPSGD-IS, DPSGD-HF, DPSGD-TS, and {DPAGD} respectively. Note that we do not compare DPSUR with those approaches that modify the structures of over-parameterized models \cite{de2022unlocking} or the semi-supervised model PATE \cite{papernot2017semi,papernot2018scalable}, as they differ significantly from the scope of this work.}
\subsubsection{Datasets and Models}
Experimental evaluation is conducted over three image classification datasets, including MNIST~\cite{lecun2010mnist}, Fashion MNIST (FMNIST)~\cite{xiao2017fashion}, and CIFAR-10~\cite{Krizhevsky2009CIFAR}, and a movie review dataset IMDB~\cite{IMDb}.

\textbf{MNIST} contains 60,000 training samples and 10,000 testing samples of handwritten digits, divided into 10 categories with 7,000 grayscale images per category. Each sample consists of a grayscale image of size $28\times28$ and a corresponding label indicating its category.
The model trained using handcrafted features as inputs achieves 99.11\% accuracy after 20 epochs in the non-private setting~\cite{tramer2020differentially}.

\textbf{FMNIST} consists of 60,000 training samples and 10,000 testing samples of fashion products categorized into 10 categories, with each category containing 7,000 grayscale images of size $28 \times 28$. The dataset also includes labels indicating the category of each image. 
The model trained using handcrafted features as inputs achieves 90.98\% accuracy after 20 epochs in the non-private setting~\cite{tramer2020differentially}.

\textbf{CIFAR-10} comprises 50,000 training samples and 10,000 testing samples of colored objects categorized into 10 categories. Each category contains 6,000 color images of size $32 \times 32$ with three color channels. Additionally, each sample is accompanied by a label indicating the category to which it belongs.
 The model trained using handcrafted features as inputs achieves 71.12\% accuracy after 20 epochs in the non-private setting~\cite{tramer2020differentially}.

\textbf{IMDb} consists of 50,000 reviews of movies, each review encoded as a list of word indexes and labeled with an obvious bias towards either positive or negative sentiment. The dataset is divided into a training set of 25,000 reviews and a test set of 25,000 reviews. In the non-private setting, the model trained using cross-entropy loss function, Adam optimizer, and an expected batch size of 32 achieves an accuracy of 79.97\% after 20 epochs.

We apply the same convolutional neural network architecture as~\cite{papernot2021tempered,tramer2020differentially,wei2022dpis} to three image datasets, i.e., MNIST, FMNIST, and CIFAR-10. Additionally, we used a same five-layer recurrent neural network as in~\cite{wei2022dpis} for the IMDB dataset. We use the categorical cross-entropy loss function for all datasets. The details of model architectures are presented in Appendix~\ref{Appendix:model}.

\begin{table}[h]
	\centering
	\caption{Noise multiplier for validation $\sigma_v$}
	\begin{tabular}{lcccc}
		\toprule
		Dataset & \text {$\epsilon=1$} & \text{$\epsilon=2$} & \text{$\epsilon=3$} & \text{$\epsilon=4$} \\
		\midrule
		MNIST & 1.3 & 1.0 & 0.9 & 0.8 \\
		FMNIST  & 1.3 & 1.3 & 0.8 & 0.8 \\
		CIFAR-10 & 1.3 & 1.3 & 1.1 & 1.1 \\
		IMDB & 1.3 & 1.2 & 1.0 & 0.9 \\
		\bottomrule
	\end{tabular}
	\label{table:The setting of noise multiplier for validation}
\end{table}

\subsubsection{Parameter Settings}
In our experiments, we set the privacy budget $\epsilon$ from $1$ to $4$ for each dataset while fixing $\delta=10^{-5}$. For image datasets, we user the SGD optimizer with a momentum parameter set to 0.9; for the IMDB dataset, we employ the Adam optimizer whose parameters are the same as~\cite{chollet2015keras}.

During the DPSGD phase, for the three image datasets, we adopt the best parameters recommended in~\cite{tramer2020differentially}. Specifically, we fine-tune the noise multiplier $\sigma_t$ for various $\epsilon$, following the approach outlined in~\cite{tramer2020differentially,wei2022dpis}. This fine-tuning process is a common practice in all privacy-preserving machine learning methods, and it does not incur any privacy loss. For the IMDB dataset, we enumerate different values and choose the best for each parameter since the competitor method~\cite{wei2022dpis} does nshiot specify them.

Based on our analysis in Section~\ref{sec:cliping} and ~\ref{sec:threshold}, we set the clipping bound $C_v=0.001$ for $\Delta E$ and acceptance parameter $\beta=-1$ to all privacy budgets $\epsilon$ and datasets. {In addition, for MNIST, FMNIST and CIFAR-10, we set the batch size of validation set $B_v=256$. While IMDB, which  have fewer training samples, we set the batch size of validation set $B_v=128$.} The noise multiplier for validation $\sigma_v$ ranges from $0.8$ to $1.3$ for all datasets and privacy budgets, as shown in Table~\ref{table:The setting of noise multiplier for validation}. Intuitively, when the privacy budget is small, we increase $\sigma_v$ to add more iteration rounds.

\subsection{Overall Performance}

\begin{table*}[t]
	\caption{{Results of classification accuracy}}
	\centering
	\label{tab:algorithm_comparison}
	\setlength{\tabcolsep}{5mm}{
		\begin{tabular}{lllccccc}
			\toprule
			Dataset & Method  & $\epsilon=1$ & $\epsilon=2$ & $\epsilon=3$ & $\epsilon=4$ & non-private \\
			\hline
			{MNIST}
			& {{\sf DPSUR}}  & \textbf{97.93\%}  & \textbf{98.70\%} &\textbf{98.88\%}  & \textbf{98.95\%} \\
			(Image Dataset) & {\sf {DPIS}~\cite{wei2022dpis}}  & {97.79\%}  & {98.51\%} &{98.62\%}  & {98.78\%} \\
			& {{\sf DPSGD-HF}~\cite{tramer2020differentially}}  & 97.78\% & 98.39\% &98.32\% &98.56\% & 99.11\%\\
			& {{\sf DPSGD-TS}~\cite{papernot2021tempered}}  & 97.06\% & 97.87\% & 98.22\% & 98.51\% \\
			& {{\sf {DPAGD}}~\cite{lee2018concentrated}}  & {95.91\%}  & {97.30\%}  & {97.52\%}  & {97.83\%}  \\
			& {{\sf DPSGD}~\cite{abadi2016deep}}  & 95.11\% & 96.10\% & 96.82\% & 97.25\% \\
			\midrule
			{FMNIST}
			& {{\sf DPSUR}}   & \textbf{88.38\%}  & \textbf{89.34\%} &\textbf{89.71\%}  & \textbf{90.18\%} \\
			(Image Dataset)& {\sf {DPIS}~\cite{wei2022dpis}}  & {86.25\%}  & {88.24\%} &{88.82\%}  & {89.21\%} \\
			& {{\sf DPSGD-HF}~\cite{tramer2020differentially}}   & 85.54\% & 87.96\% & 89.01\% & 89.06\% & 90.98\%\\
			& {{\sf DPSGD-TS}~\cite{papernot2021tempered}}  & 83.63\% & 85.33\% &86.29\% & 86.86\% \\
			& {{\sf {DPAGD}}~\cite{lee2018concentrated}}  & {81.26\%}  & {84.50\%}  & {86.04\%}  & {86.78\%}  \\
			& {{\sf DPSGD}~\cite{abadi2016deep}} & 80.25\% & 82.63\% &84.72\% & 85.40\% \\
			\midrule
			{CIFAR-10}
			& {{\sf DPSUR}}   & \textbf{64.41\%}  & \textbf{69.40\%} &\textbf{70.83\%}  & \textbf{71.45\%} \\
			(Image Dataset)& {\sf {DPIS}~\cite{wei2022dpis}} & {63.23\%}  & {67.94\%} &{69.63\%} & {70.55\%} \\
			& {{\sf DPSGD-HF}~\cite{tramer2020differentially}}   & 63.15\% & 66.55\% & 69.35\% & 70.28\% & 71.12\%\\
			& {{\sf DPSGD-TS}~\cite{papernot2021tempered}}   & 51.52\% & 56.78\% & 60.42\% & 61.75\% \\
			& {{\sf {DPAGD}}~\cite{lee2018concentrated}}  & {45.78\%}  & {53.30\%}  & {56.21\%}  & {60.31\%}  \\
			& {{\sf DPSGD}~\cite{abadi2016deep}} & 46.03\% & 51.33\% & 54.67\% & 58.89\% \\
			\midrule
			{IMDb}
			& {{\sf DPSUR}}   & \textbf{66.50\%}  & \textbf{71.02\%} &\textbf{72.16\%}  & \textbf{74.14\%} \\
			(Text Dataset)&{\sf {DPIS}~\cite{wei2022dpis}}     & {63.56\%}  & {66.11\%} &{68.49\%} &{70.12\%} \\
			
			& {{\sf DPSGD-TS}~\cite{papernot2021tempered}} & 65.08\% & 68.34\% &70.10\% & 70.85\% & 79.97\%\\
			& {{\sf {DPAGD}}~\cite{lee2018concentrated}}  & {58.72\%}  & {63.48\%}  & {64.59\%}  & {66.01\%}  \\
			& {{\sf DPSGD}~\cite{abadi2016deep}}  & 64.13\% & 68.55\% &70.41\% & 71.57\%     \\
			\bottomrule
	\end{tabular}}
	\vspace{2mm}
	
\end{table*}

Table~\ref{tab:algorithm_comparison} shows the classification accuracies of DPSUR and five competitive methods.\footnote{Since the scattering network used in DPSGD-HF~\cite{tramer2020differentially} is not applicable to natural language processing, we omit the results of this method on the IMDb dataset.} It is noteworthy that DPSUR consistently outperforms all competitors across all datasets and privacy budgets, except for a less eminent advantage on the MNIST dataset where the accuracy of \cite{wei2022dpis} already approaches that of the non-private setting, leaving little room for further improvement. For the other three datasets, the classification accuracy of DPSUR is at least 1\% higher than the second best, which shows a huge improvement over DPSGD.

Notably, DPSUR performs almost as well as in the non-private setting at $\epsilon=4$ in three image datasets. The superior performance of the DPSUR is attributed to our objective of selecting model updates to minimize the loss function. Moreover, We derive the RDP for the selective Gaussian mechanism, which allows us to reduce the consumption of privacy loss. In particular, on the CIFAR-10 dataset, we observe that DPSUR even outperforms non-private results when $\epsilon=4$. This is because moderate noise in SGD sometimes helps the neural network escape from local minima~\cite{ge2015escaping}.

\begin{figure*}[htb]
	\centering
	\begin{subfigure}{0.19\linewidth}
		\centering
		\includegraphics[width=1.0\linewidth]{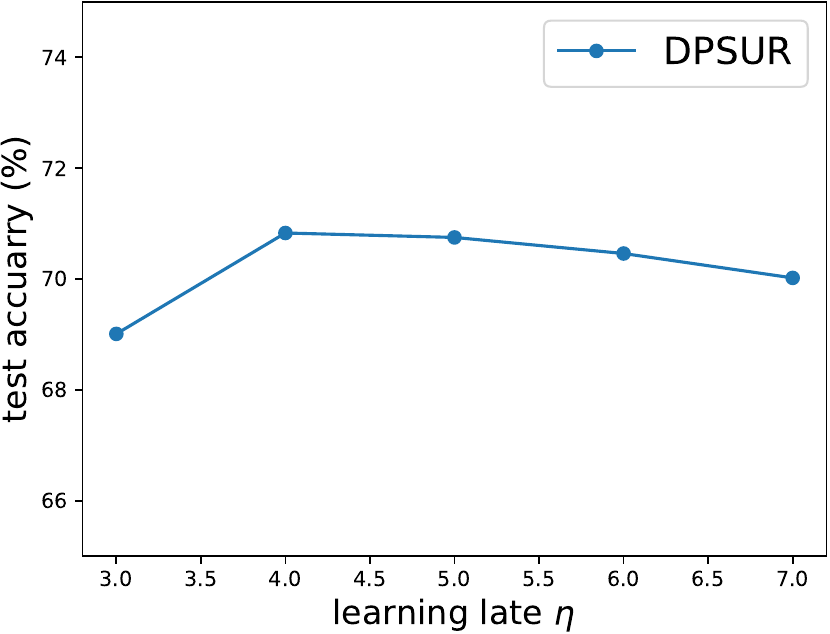}
		\caption{Impact of $\eta$}
		\label{figure:Ablation a}
	\end{subfigure}
	\centering
	\begin{subfigure}{0.19\linewidth}
		\centering
		\includegraphics[width=1.0\linewidth]{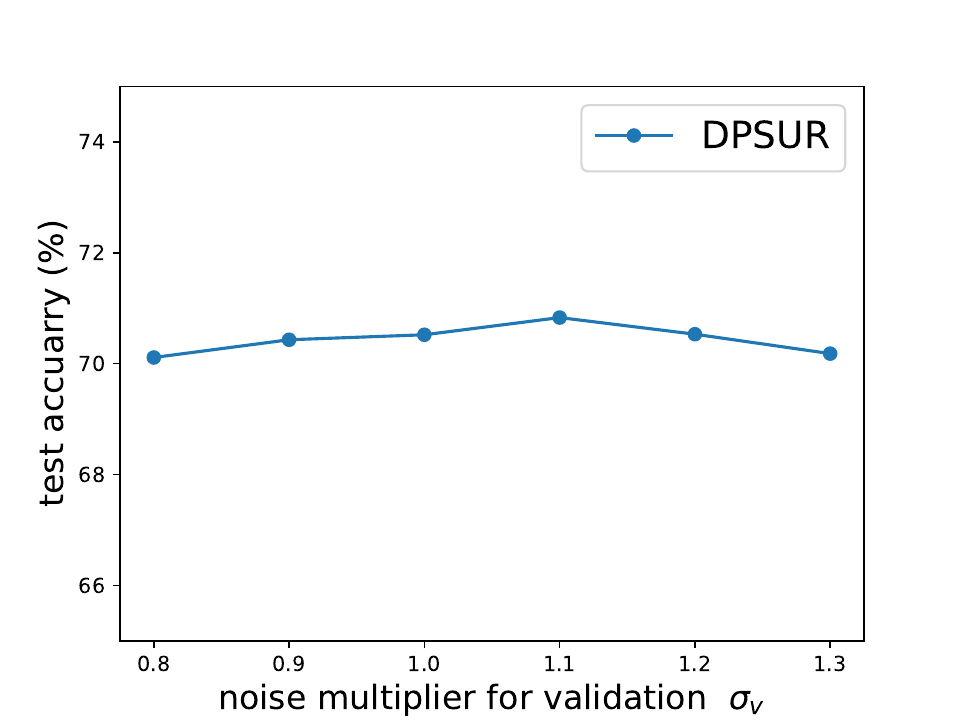}
		\caption{Impact of $\sigma_v$}
		\label{figure:Ablation b}
	\end{subfigure}
	\centering
	\begin{subfigure}{0.19\linewidth}
		\centering
		\includegraphics[width=1.0\linewidth]{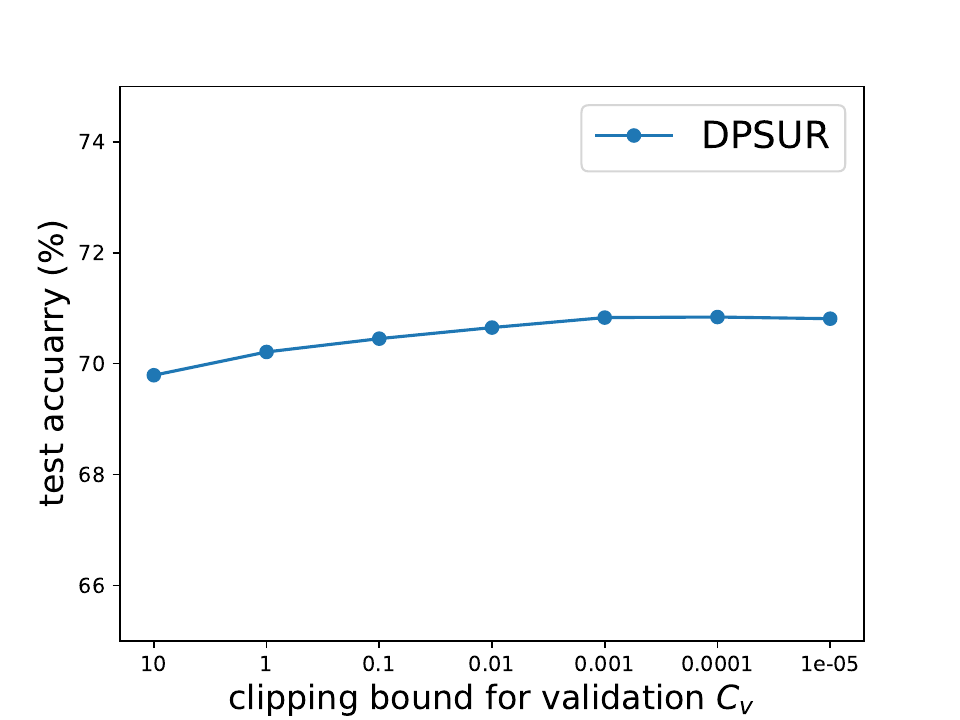}
		\caption{Impact of $C_v$}
		\label{figure:Ablation c}
	\end{subfigure}
	\centering
	\begin{subfigure}{0.19\linewidth}
		\centering
		\includegraphics[width=1.0\linewidth]{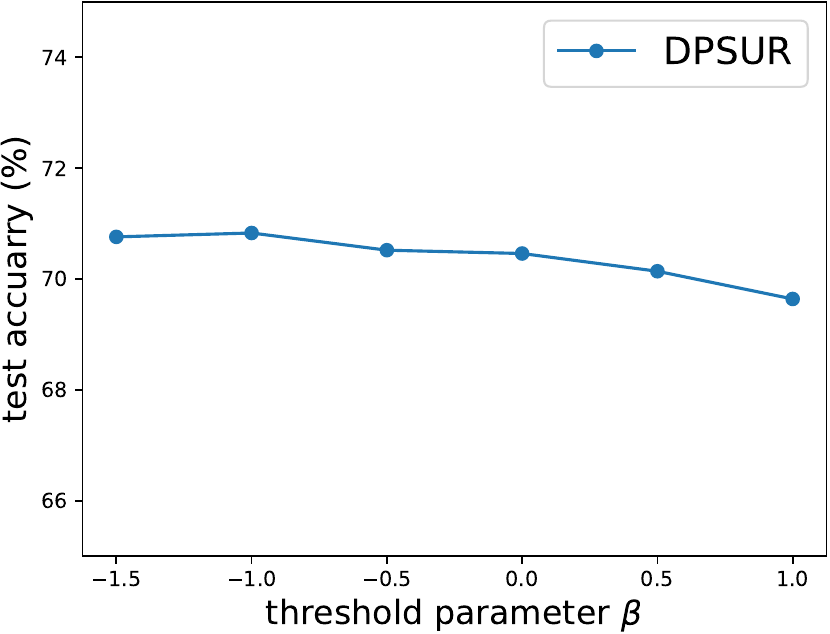}
		\caption{Impact of $\beta$}
		\label{figure:Ablation d}
	\end{subfigure}
	\centering
	\begin{subfigure}{0.19\linewidth}
		\centering
		\includegraphics[width=1.0\linewidth]{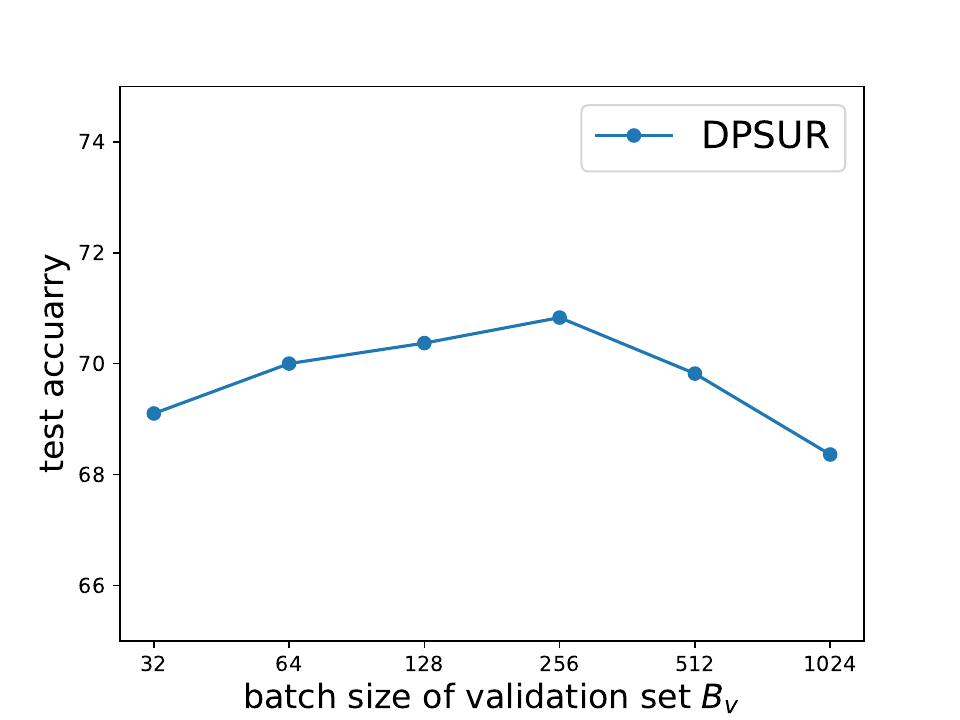}
		\caption{{Impact of $B_v$}}
		\label{figure:Ablation e}
	\end{subfigure}
	\caption{The impact of different parameters on the test accuracy in CIFAR-10.}
	\label{figure:Ablation}
\end{figure*}

\subsection{Impact of Various Parameters}
In this subsection, we study the impact of various parameters of DPSUR, including the learning rate, the noise multiplier of validation, the cilpping bound of loss, and the threshold parameter. Due to space limitation, we only show the results of CIFAR-10 dataset. In all experiments, if not specified, we use the SGD optimizer with the momentum $0.9$, and set the learning rate $\eta= 4.0$, batch size for training $B_t = 8192$, noise multiplier for training $\sigma_t = 5.67$, batch size for validation $B_v=128$, noise multiplier for validation $\sigma_v=1.1$, clipping bound for validation $C_v=0.001$, the threshold parameter $\beta=-1$, and privacy budget $(3,10^{-5})$.

\textbf{Learning rate $\eta$.}
As plotted in Figure~\ref{figure:Ablation a}, the highest accuracy achieved is 70.83\% when using a learning rate of $\eta=4$. If the learning rate is larger than $4$, the accuracy starts to decrease. This is because a large learning rate may cause slow convergence to the trained model. However, thanks to selective update, DPSUR becomes adaptive to different learning rates, and the accuracy remains stable at 70.02\% when using a extremely high learning rate $\eta=7$. 


\textbf{Noise multiplier for validation $\sigma_v$.}
A large $\sigma_v$ can save privacy budget running more rounds, which degrades the quality of the accepted model. On the contrary, a small $\sigma_v$ can guarantee the accepted model quality, but consumes more privacy budget. According to Figure~\ref{figure:Ablation b}, we find that the test accuracy of DPSUR is quite stable at around $70.3\%$ when $\sigma_v$ increases from 0.8 to 1.3, and the highest accuracy is $70.83\%$ when $\sigma_v=1.1$.


\textbf{Clipping bound for validation $C_v$.}
As aforementioned, using a sufficiently small clipping bound $C_v$ can discretize the difference of the loss $\Delta E$. With a shallow model training on CIFAR-10, loss values ranging from $0.01$ to $0.0001$ are all considered sufficiently small, a smaller $C_v$ (e.g. $1e-05$) than the front does not bring performance gains. As shown in Figure~\ref{figure:Ablation c}, the accuracy for the cases of $C_v \in [0.01, 1e-05]$ achieves slightly better performance. However, when $C_v$ is set to 1 or 10, we observe a significant decline in performance, which is consistent with our analysis that setting a large $C_v$ does not provide any performance benefits.

\textbf{Threshold parameter $\beta$.}
The acceptance probability is influenced by the parameter $\beta$, with smaller values leading to a decreasing probability of accepting both low-quality and high-quality models. As shown in Figure~\ref{figure:Ablation d}, the best performance is achieved when $\beta=-1.0$. This is consistent with our analysis in Section~\ref{sec:threshold} that a smaller $\beta$ leads to a higher rejection probability for low-quality updates, thereby guiding the model towards the correct direction during iterations. It is worth noting that setting $\beta$ too small does not gain more benefits, as a very small $\beta$  causes the model to reject almost all high-quality and low-quality solutions, contributing nothing to the model convergence.

{\textbf{Batch size of validation set $B_v$.}}
{A small $B_v$ can help conserve the privacy budget, enabling more rounds of computation. However, this can result in a decrease in the quality of the accepted model. Conversely, a larger value of $B_v$ can ensure better model quality but consumes a greater portion of the privacy budget. As shown in Figure~\ref{figure:Ablation e}, we observe that as $B_v$ increases from 32 to 256, the test accuracy of DPSUR improves from $69.10\%$ to $70.83\%$, but it declines to $68.36\%$ when $B_v=1024$.}

\subsection{Resilience Against Member Inference Attacks}\label{sec:Resilience Against Member Inference Attacks}
Differential privacy protection is naturally resistant to membership inference attacks. To empirically verify if DPSUR achieves the same privacy guarantee as DPSGD, we conduct membership inference attacks on models trained on FMNIST and CIFAR-10, where their models are trained from DPSUR and DPSGD algorithms, respectively.

\subsubsection{Attack overview}
{We adopt two membership inference attacks, Black-Box/Shadow~\cite{salem2019ml} and White-Box/Partial~\cite{nasr2019comprehensive}, which are the SOTA methods in membership inference attack to our knowledge.}

\textbf{Black-Box/Shadow.}
In the Black-Box/Shadow attack scenario, the adversary has a shadow auxiliary dataset. The dataset is divided into two parts, with one part used to train a shadow model for the same task. The shadow model is then queried using the entire shadow dataset.
For each query sample, the shadow model provides its posterior probability and predicted label. The adversary labels the sample as a member if it belongs to the training set of the shadow model, otherwise, it is labeled as a non-member. Using this labeled dataset, the adversary trains an attack model, which serves as a binary classifier to distinguish between members and non-members.
To determine if a sample belongs to the target model's training dataset, it is inputted into the target model for prediction. The resulting posterior probability and predicted label (converted into a binary indicator of prediction correctness) are then fed into the attack model.

\textbf{White-Box/Partial.}
{In the White-Box/Partial attack scenario, the adversary has partial training dataset as the auxiliary dataset. One advantage in the White-Box/Partial attack is that the adversary has access to the target model. This allows adversary to utilize various resources, including gradients with respect to model parameters, embeddings from intermediate layers, classification loss, as well as posterior probabilities and labels of the target samples.}

\subsubsection{Attack setting}

For each dataset, we randomly split it into four subsets: the target training dataset, target testing dataset, shadow training dataset, and shadow testing dataset. The ratio of the sample sizes in each subset is 2:1:2:1. Our target model and training parameters are consistent with those described above.

\subsubsection{Results.}
Table~\ref{tab:Results on Accuracy of Member Inference Attack in FashionMnist} and Table~\ref{tab:Results on Accuracy of Member Inference Attack in CIFAR-10} report the accuracy of two inference attacks against target models protected by DPSUR and DPSGD on FMNIST and CIFAR-10, respectively. {We observe that member inference attacks are quite effective against the non-private methods (non-dp), especially on CIFRA-10. As for the two DP algorithms, the attack accuracy drops from 0.58 to around 0.50 on the FMNIST, and drops from 0.73 to around 0.50 on the CIFAR-10, which almost equals to random guess. It's noteworthy that the attack performance on FMNIST is consistently poor, as models trained on FMNIST generalize well on non-member data samples~\cite{shokri2017membership}.} These results show that the model under DP protection can defend very well against membership inference attacks, and our DPSUR algorithm can provide the same level of privacy protection as DPSGD.

\begin{table}[h]
	\centering~
	\caption{{Accuracy of Member Inference Attack on FMNIST}}
	\small
	\begin{tabular}{lcccccc}		
		\toprule
		Attack & Algorithm &$\epsilon=1$ &  $\epsilon=2$ & $\epsilon=3$ & $\epsilon=4$ & non-private\\
		\midrule
		\multirow{2}*{\parbox{1.2cm}{Black-Box/Shadow}}
		& DPSUR & 0.498  & 0.500 & 0.503 & 0.506 & \multirow{2}*{0.582} \\
		& DPSGD & 0.498  & 0.503 & 0.493 & 0.494\\
		\midrule 
		\multirow{2}*{\parbox{1.2cm}{White-Box/Partial}}
		& DPSUR & {0.499}   & {0.504} & {0.501} & {0.502} & \multirow{2}*{{0.584}}\\
        & DPSGD & {0.501} & {0.502} & {0.502} & {0.505} \\
		\bottomrule
	\end{tabular}
	\label{tab:Results on Accuracy of Member Inference Attack in FashionMnist}
\end{table}

\begin{table}[h]
	\centering
	\caption{{Accuracy of Member Inference Attack on CIFAR-10}}
	\small
	\begin{tabular}{lcccccc}		
		\toprule
		Attack & Algorithm & $\epsilon=1$ &  $\epsilon=2$ & $\epsilon=3$ & $\epsilon=4$ & non-private\\
		\midrule
		\multirow{2}*{\parbox{1.2cm}{Black-Box/Shadow}}
		& DPSUR& 0.495  & 0.498 & 0.503 & 0.504 & \multirow{2}*{0.732} \\
		& DPSGD & 0.504  & 0.505 & 0.504 & 0.505\\
		\midrule
		\multirow{2}*{\parbox{1.2cm}{White-Box/Partial}}\color{red}
		& DPSUR& {0.499}  & {0.501} & {0.502} & {0.503} & \multirow{2}*{{0.743}}\\
        & DPSGD & {0.500}  & {0.501} & {0.501} & {0.503} \\ 
		\bottomrule
	\end{tabular}
	\label{tab:Results on Accuracy of Member Inference Attack in CIFAR-10}
\end{table}

%% file: related.tex
\section{Related Work}
\label{sec:related}
Privacy-preserving model training was first proposed in~\cite{song2013stochastic,bassily2014private}.
Subsequently, Abadi et al. \cite{abadi2016deep} proposed a generalized algorithm, DPSGD, for deep learning with differential privacy, and since then many works aimed at improving DPSGD from different aspects.

\textbf{Gradient clipping.}
At each iteration of training, Zhang et al.~\cite{zhang2018differentially} used public data to obtain an approximate bound on gradient norm and clip the gradients at this approximate bound. The work in~\cite{Veen_Seggers_Bloem_Patrini_2018} proposed adaptive clipping in each layer of the neural network. {Andrew et al.~\cite{andrew2021differentially} designed a method for adaptively tuning the clipping threshold to track a given quantile of the update norm distribution during training, especially in federated learning.} Venkatadheeraj et al. \cite{pichapati2019adaclip} proposed AdaCliP, which using coordinate-wise adaptive clipping of the gradient.
However, a recent work~\cite{yang2022normalized} has shown that by redefining the clipping equation as $Clip_C(g)=C/||g||_2$, clipping is actually equivalent to normalization by setting the clipping bound small enough. 

{Our paper does not employ any adaptive clipping technique during the training phase. Instead, our minimal clipping is orthogonal to theirs as it is for threshold evaluation in the validation phase, not the training phase.}

\textbf{Gaussian noise.}
The work of \cite{phan2017adaptive} implemented adaptive noise addition using a hierarchical correlation propagation protocol approach, adding a small amount of noise to features with high correlation to the model's output. Balle and Wang~\cite{balle2018improving} introduced an optimized Gaussian mechanism that directly calibrates variance using the Gaussian cumulative density function instead of relying on a tail-bound approximation. Their work is orthogonal to ours and can be incorporated ours as the underlying perturbation mechanism. Lee et al.~\cite{lee2018concentrated} selected the best learning rate from a candidate set based on model evaluation and implemented adaptive privacy budget allocation in each round of DPSGD training. While both our work and theirs involve adaptive DPSGD, there are three significant distinctions. First, their adaptiveness is from adaptive learning rates, which modifies the gradient descent's step size but not its direction. Second, the DP privacy guarantee is different. Ours relies on differences in loss values before and after iterations, whereas theirs adopts NoisyMax\cite{Dwork:2014:AFD:2693052.2693053} method on loss values from a variety of models obtained through different learning rates in a single iteration. In contrast, we aim to ensure each gradient descent moves in the desired direction through resampling and re-noising. Third, we do not perform adaptive privacy budget allocation but instead introduce selective release techniques to preserve the privacy budget.Further, Xu et al. \cite{xu2020adaptive} employed the Root Mean Square Prop (RMSProp) gradient descent technique to adaptively add noise to coordinates of the gradient. Since then, many works \cite{golatkar2022mixed,zhou2020bypassing,yu2020not} have focused on reducing the dimensionality of the model during training to reduce the impact of noise on the overall model.

\textbf{Poisson sampling.}
{Wei et al. \cite{wei2022dpis} first explored the problem of bias due to Poisson sampling in DPSGD and proposed DPIS, which weights the importance sampling by the gradient norm of the sample. Our algorithm mitigates the impact of Poisson sampling on convergence speed through the process of resampling.
}

\textbf{Models, pre-processing and parameter tuning.}
Papernot et al. \cite{papernot2021tempered} found that using a family of bounded activation functions (tempered sigmoids) instead of the unbounded activation function ReLU in DPSGD can achieve good performance. Tramer et al. \cite{tramer2020differentially} used Scattering Network to traverse the image in advance to extract features before DPSGD training. Soham et al. \cite{de2022unlocking} combined careful hyper-parameter tuning with group normalization and weight standardization to yield remarkable performance benefits. These works are orthogonal to our work.

\textbf{Privacy accounting.}
Abadi et al.~\cite{abadi2016deep} proposed a method called the Moments Accountant (MA) for giving an upper bound the privacy curve of a composition of DPSGD. The Moments Accountant was subsumed into the framework of Renyi Differential Privacy (RDP) introduced by \cite{mironov2017renyi}.
Bu et al. \cite{bu2020deep} introduced the notion of Gaussian Differential Privacy (GDP) base hypothesis test. There also exits other variants of DP, for example Concentrated DP (CDP) \cite{Bun_Steinke_2016} and zero Concentrated-DP \cite{Bun_Steinke_2016}. {These variants are tailored for specific scenarios and can be converted into one another under certain conditions. Our primary focus is on $(\epsilon,\delta)$-differential privacy, as it is the most prevalent and widely adopted in both academic literature and practical applications. Besides, many works~\cite{Ye_Hu_Meng_Zheng_Huang_Fang_Shi_2023, zhang2023trajectory, du2023differential, ye2023stateful, duan2022utility} based on local differential privacy focus on $\epsilon$- differential privacy.}

%% file: conclusion.tex
\section{Conclusion}
\label{sec:conclusion}
We propose DPSUR, a differentially private scheme for deep learning based on selective update and release. Our scheme utilizes the validation test to select appropriate model updates in each iteration, thereby speeding up model convergence and enhancing utility. To reduce the injected Gaussian noise, we incorporate a clipping strategy and a threshold mechanism for gradient selection in each iteration. Furthermore, we apply the Gaussian mechanism with selective release to reduce privacy budget consumption across iterations. We conduct a comprehensive privacy analysis of our approach using RDP and validate our scheme through extensive experiments. The results indicate that DPSUR significantly outperforms state-of-the-art solutions in terms of model utility and downstream tasks. Our scheme is widely applicable to various neural networks, and can serve as a flexible optimizer for new DPSGD variants. For future work, we plan to extend DPSUR to larger models and datasets and theoretically analyze its convergence speed.


%% file: appendix.tex
\section{RDP of two truncated normal distributions} \label{Appendix:Rényi Divergence of Truncated Gaussian distribution}

\begin{align}
\begin{split} 
\begin{aligned}
D_{\alpha}( & f(x ; \mu, \mu\sigma, a, b )||f(x ; 0, \mu\sigma, a, b)) \\
= & \frac{1}{\alpha-1} \cdot \log \int_{a}^{b} \frac{[f(x ; \mu, \mu\sigma, a, b )]^\alpha}{[f(x ; 0, \mu\sigma, a, b)]^{\alpha-1}} \mathrm{d} x \\
= & \frac{1}{\alpha-1} \cdot \log \{\frac{(\Phi(\frac{b}{\mu\sigma })-\Phi(\frac{a}{\mu\sigma } ))^{\alpha-1}  }{(\Phi(\frac{b-\mu}{\mu\sigma })-\Phi(\frac{a-\mu}{\mu\sigma }))^{\alpha}} \cdot \int_{a}^{b} \frac{1}{\mu\sigma \sqrt{2 \pi}} \\
&  \cdot \exp (\frac{-\alpha (x-\mu)^2}{2 \mu^2\sigma^2}) \cdot \exp (\frac{-(1-\alpha) x^{2}}{2 \mu^2\sigma^2}) \mathrm{d} x \}\\
= & \frac{1}{\alpha-1}\cdot \log \{\frac{(\Phi(\frac{b}{\mu\sigma })-\Phi(\frac{a}{\mu\sigma } ))^{\alpha-1}  }{(\Phi(\frac{b-\mu}{\mu\sigma })-\Phi(\frac{a-\mu}{\mu\sigma }))^{\alpha}} \cdot \frac{1}{\mu\sigma \sqrt{2 \pi}} \int_{a}^{b} \exp \left[\left(-x^{2}+\right.\right. \\
& \left.\left.2 \alpha \mu x- \alpha \mu^{2}\right) /\left(2 \mu^2\sigma^2\right)\right] \mathrm{d} x \}\\
= & \frac{1}{\alpha-1} \cdot \log \{\frac{(\Phi(\frac{b}{\mu\sigma })-\Phi(\frac{a}{\mu\sigma } ))^{\alpha-1}  }{(\Phi(\frac{b-\mu}{\mu\sigma })-\Phi(\frac{a-\mu}{\mu\sigma }))^{\alpha}} \cdot \frac{1}{\sqrt{\pi}} \int_{a}^{b} exp(\frac{\alpha(\alpha -1)}{2\sigma^2 })  \\
& \cdot exp(-(\frac{x-\alpha \mu}{\sqrt[]{2}\mu\sigma })^2 )\mathrm{d}(\frac{x-\alpha \mu}{\sqrt[]{2}\mu\sigma })\} \\
= & \frac{1}{\alpha-1} \cdot \{ \frac{\alpha(\alpha-1)}{2 \sigma^2} + \log [\frac{(\Phi(\frac{b}{\mu\sigma })-\Phi(\frac{a}{\mu\sigma } ))^{\alpha-1}  }{(\Phi(\frac{b-\mu}{\mu\sigma })-\Phi(\frac{a-\mu}{\mu\sigma }))^{\alpha}} \cdot \frac{1}{\sqrt{\pi}}  \\
&  \cdot \int_{a}^{b} exp(-(\frac{x-\alpha \mu}{\sqrt[]{2}\mu\sigma })^2 )\mathrm{d}(\frac{x-\alpha \mu}{\sqrt[]{2}\mu\sigma })]\} \\
= & \frac{1}{\alpha-1} \cdot \{ \frac{\alpha(\alpha-1)}{2 \sigma^2} + \log [\frac{(\Phi(\frac{b}{\mu\sigma })-\Phi(\frac{a}{\mu\sigma } ))^{\alpha-1}  }{(\Phi(\frac{b-\mu}{\mu\sigma })-\Phi(\frac{a-\mu}{\mu\sigma }))^{\alpha}}  \\
& \cdot  \int_{\frac{a-\alpha \mu}{\mu \alpha}}^{\frac{b-\alpha \mu}{\mu \alpha}} \frac{1}{\sqrt{2\pi}} \cdot exp(-\frac{x^2}{2} )\mathrm{d}(x)]\} \\
= & \frac{1}{\alpha-1} \cdot \{\frac{\alpha(\alpha-1)}{2\sigma ^2} + \log [ \frac{(\Phi(\frac{b}{\mu\sigma })-\Phi(\frac{a}{\mu\sigma } ))^{\alpha-1}  }{(\Phi(\frac{b-\mu}{\mu\sigma })-\Phi(\frac{a-\mu}{\mu\sigma }))^{\alpha}} \\
& \cdot ({\Phi(\frac{b-\alpha \mu}{\mu\sigma})-\Phi(\frac{a-\alpha \mu}{\mu\sigma})})]\} \\
= & \frac{\alpha }{2 \sigma^2} + \frac{1}{\alpha-1} \cdot \log\{\frac{(\Phi(\frac{b}{\mu\sigma })-\Phi(\frac{a}{\mu\sigma } ))^{\alpha-1}  }{(\Phi(\frac{b-\mu}{\mu\sigma })-\Phi(\frac{a-\mu}{\mu\sigma }))^{\alpha}} \cdot [\Phi(\frac{b-\alpha \mu}{\mu\sigma})\\
& -\Phi(\frac{a-\alpha \mu}{\mu\sigma})]\},\\
& \textrm{where} \; \Phi(x)=\frac{1}{\sqrt{2 \pi}} \int_{-\infty}^{x} e^{-\frac{t^{2}}{2}} dt \nonumber.
\end{aligned}
\end{split}
\end{align}

\begin{align}
\begin{split} 
\begin{aligned}
D_{\alpha}( & f(x ; \mu, \mu\sigma, a, b )||f(x ; 0, \mu\sigma, a, b)) \\
= & \frac{1}{\alpha-1} \cdot \ln \int_{a}^{b} \frac{[f(x ; \mu, \mu\sigma, a, b )]^\alpha}{[f(x ; 0, \mu\sigma, a, b)]^{\alpha-1}} \mathrm{d} x \\
= & \frac{1}{\alpha-1} \cdot \ln \{\frac{(\Phi(\frac{b}{\mu\sigma })-\Phi(\frac{a}{\mu\sigma } ))^{\alpha-1}  }{(\Phi(\frac{b-\mu}{\mu\sigma })-\Phi(\frac{a-\mu}{\mu\sigma }))^{\alpha}} \cdot \int_{a}^{b} \frac{1}{\mu\sigma \sqrt{2 \pi}} \\
&  \cdot \exp (\frac{-\alpha (x-\mu)^2}{2 \mu^2\sigma^2}) \cdot \exp (\frac{-(1-\alpha) x^{2}}{2 \mu^2\sigma^2}) \mathrm{d} x \}\\
= & \frac{1}{\alpha-1}\cdot \ln \{\frac{(\Phi(\frac{b}{\mu\sigma })-\Phi(\frac{a}{\mu\sigma } ))^{\alpha-1}  }{(\Phi(\frac{b-\mu}{\mu\sigma })-\Phi(\frac{a-\mu}{\mu\sigma }))^{\alpha}} \cdot \frac{1}{\mu\sigma \sqrt{2 \pi}} \int_{a}^{b} \exp \left[\left(-x^{2}+\right.\right. \\
& \left.\left.2 \alpha \mu x- \alpha \mu^{2}\right) /\left(2 \mu^2\sigma^2\right)\right] \mathrm{d} x \}\\
= & \frac{1}{\alpha-1} \cdot \ln \{\frac{(\Phi(\frac{b}{\mu\sigma })-\Phi(\frac{a}{\mu\sigma } ))^{\alpha-1}  }{(\Phi(\frac{b-\mu}{\mu\sigma })-\Phi(\frac{a-\mu}{\mu\sigma }))^{\alpha}} \cdot \frac{1}{\sqrt{\pi}} \int_{a}^{b} exp(\frac{\alpha(\alpha -1)}{2\sigma^2 })  \\
& \cdot exp(-(\frac{x-\alpha \mu}{\sqrt[]{2}\mu\sigma })^2 )\mathrm{d}(\frac{x-\alpha \mu}{\sqrt[]{2}\mu\sigma })\} \\
= & \frac{1}{\alpha-1} \cdot \{ \frac{\alpha(\alpha-1)}{2 \sigma^2} + \ln [\frac{(\Phi(\frac{b}{\mu\sigma })-\Phi(\frac{a}{\mu\sigma } ))^{\alpha-1}  }{(\Phi(\frac{b-\mu}{\mu\sigma })-\Phi(\frac{a-\mu}{\mu\sigma }))^{\alpha}} \cdot \frac{1}{\sqrt{\pi}}  \\
&  \cdot \int_{a}^{b} exp(-(\frac{x-\alpha \mu}{\sqrt[]{2}\mu\sigma })^2 )\mathrm{d}(\frac{x-\alpha \mu}{\sqrt[]{2}\mu\sigma })]\} \\
= & \frac{1}{\alpha-1} \cdot \{ \frac{\alpha(\alpha-1)}{2 \sigma^2} + \ln [\frac{(\Phi(\frac{b}{\mu\sigma })-\Phi(\frac{a}{\mu\sigma } ))^{\alpha-1}  }{(\Phi(\frac{b-\mu}{\mu\sigma })-\Phi(\frac{a-\mu}{\mu\sigma }))^{\alpha}}  \\
& \cdot  \int_{\frac{a-\alpha \mu}{\mu \alpha}}^{\frac{b-\alpha \mu}{\mu \alpha}} \frac{1}{\sqrt{2\pi}} \cdot exp(-\frac{x^2}{2} )\mathrm{d}(x)]\} \\
= & \frac{1}{\alpha-1} \cdot \{\frac{\alpha(\alpha-1)}{2\sigma ^2} + \ln [ \frac{(\Phi(\frac{b}{\mu\sigma })-\Phi(\frac{a}{\mu\sigma } ))^{\alpha-1}  }{(\Phi(\frac{b-\mu}{\mu\sigma })-\Phi(\frac{a-\mu}{\mu\sigma }))^{\alpha}} \\
& \cdot ({\Phi(\frac{b-\alpha \mu}{\mu\sigma})-\Phi(\frac{a-\alpha \mu}{\mu\sigma})})]\} \\
= & \frac{\alpha }{2 \sigma^2} + \frac{1}{\alpha-1} \cdot \ln\{\frac{(\Phi(\frac{b}{\mu\sigma })-\Phi(\frac{a}{\mu\sigma } ))^{\alpha-1}  }{(\Phi(\frac{b-\mu}{\mu\sigma })-\Phi(\frac{a-\mu}{\mu\sigma }))^{\alpha}} \cdot [\Phi(\frac{b-\alpha \mu}{\mu\sigma})\\
& -\Phi(\frac{a-\alpha \mu}{\mu\sigma})]\},\\
& \textrm{where} \; \Phi(x)=\frac{1}{\sqrt{2 \pi}} \int_{-\infty}^{x} e^{-\frac{t^{2}}{2}} dt \nonumber.
\end{aligned}
\end{split}
\end{align}

\section{Model Architectures}
\label{Appendix:model}

\begin{table}[htb]
	\centering
	\caption{MNIST and FMNIST model architecture}
	\begin{tabular}{@{}ll@{}}
		\toprule
		Layer & \text {Parameters} \\
		\midrule
		Convolution & \text {16 filters of 8$\times$8, stride 2,padding 2} \\
		Max-Pooling & \text {2$\times$2, stride 1 } \\
		Convolution & \text {32 filters of 4$\times$4, stride 2, padding 0} \\
		Max-Pooling & \text {2$\times$2, stride 1 } \\
		Fully connected & \text {32 units} \\
		Fully connected & \text {10 units} \\
		\bottomrule
	\end{tabular}
	\label{table:MNIST and FMNIST model}
\end{table}

\begin{table}[htb]
	\centering
	\caption{ CIFAR-10 model architecture}
	\begin{tabular}{@{}ll@{}}
		\toprule
		Layer & \text {Parameters} \\
		\midrule
		Convolution$\times$2 & \text {32 filters of 3$\times$3, stride 1, padding 1} \\
		Max-Pooling & \text {2$\times$2, stride 2 } \\
		Convolution$\times$2 & \text {64 filters of 3$\times$3, stride 1, padding 1} \\
		Max-Pooling & \text {2$\times$2, stride 2 } \\
		Convolution$\times$2 & \text {128 filters of 3$\times$3, stride 1, padding 1} \\
		Max-Pooling & \text {2$\times$2, stride 2 } \\
		Fully connected & \text {128 units} \\
		Fully connected & \text {10 units} \\
		\bottomrule
	\end{tabular}
	\label{table:CIFAR-10 model}
\end{table}

\begin{table}[htb]
	\caption{IMDb model architecture.}
	\centering
	\begin{tabular}{l|l}
		\toprule
		Layer & Parameters \\
		\midrule
		Embedding & 100 units \\
		Fully connected & 32 units \\
		Bidirectional LSTM & 32 units \\
		Fully connected & 16 units \\
		Fully connected & 2 units \\
		\bottomrule
	\end{tabular}
	\label{table:IMDb model}
\end{table}